\pdfoutput=1

\documentclass[11pt]{article}

\usepackage[final]{acl}
\usepackage{times}
\usepackage{latexsym}
\usepackage[noend]{algorithmic} 
\usepackage{helvet}  
\usepackage{courier}  
\usepackage{soul}
\usepackage{enumitem}
\usepackage{bm}
\usepackage[utf8]{inputenc}
\usepackage{amsmath}
\usepackage{amsthm}
\usepackage{booktabs}
\usepackage{algorithm}
\usepackage{algorithmic}
\usepackage[switch]{lineno}
\usepackage{amssymb}
\usepackage{subcaption} 
\usepackage{adjustbox}
\usepackage{multirow}
\newtheorem{theorem}{Theorem}
\usepackage{float}
\usepackage{natbib}  
\usepackage{caption} 
\usepackage{algorithm}
\usepackage{algorithmic}
\usepackage[T1]{fontenc}

\usepackage[utf8]{inputenc}

\usepackage{microtype}

\usepackage{inconsolata}

\usepackage{graphicx}

\title{Mixture of Length and Pruning Experts for Knowledge Graphs Reasoning}

\author{Enjun Du, Siyi Liu, Yongqi Zhang\thanks{~Corresponding author} \\
  The Hong Kong University of Science and Technology (Guangzhou) \\
  Hong Kong, SAR, China \\
  \texttt{\{EnjunDu.cs, ssui.liu1022\}@gmail.com, yzhangee@connect.ust.hk}}

\begin{document}
    \maketitle
\begin{abstract}
Knowledge Graph (KG) reasoning, which aims to infer new facts from structured knowledge repositories, plays a vital role in Natural Language Processing (NLP) systems. Its effectiveness critically depends on constructing informative and contextually relevant reasoning paths. However, existing graph neural networks (GNNs) often adopt rigid, query-agnostic path-exploration strategies, limiting their ability to adapt to diverse linguistic contexts and semantic nuances.
To address these limitations, we propose \textbf{MoKGR}, a mixture-of-experts framework that personalizes path exploration through two complementary components: (1) a mixture of length experts that adaptively selects and weights candidate path lengths according to query complexity, providing query-specific reasoning depth; and (2) a mixture of pruning experts that evaluates candidate paths from a complementary perspective, retaining the most informative paths for each query.
Through comprehensive experiments on diverse benchmark, MoKGR demonstrates superior performance in both transductive and inductive settings, validating the effectiveness of personalized path exploration in KGs reasoning.
\end{abstract}

\section{Introduction} \label{Introduction}

Knowledge Graphs (KGs) are integral to Natural Language Processing (NLP), offering structured knowledge representations crucial for various language understanding and generation tasks~\cite{ji2021survey, liang2024survey}. In KGs, entities and their semantic associations are systematically encoded as relational triples~(\textit{subject, relation, object})~\cite{ali2022knowledge,sun2021rotate}, often derived from or used to interpret textual data. These triples form semantic networks that capture intricate connectivity and meaning~\cite{nickel2015review,ji2022survey,wang2023enhancing}, thereby enabling advanced NLP applications like sophisticated question answering and semantic search. A KG query can be formulated via a function $\mathcal{Q}$ as $\mathcal{Q}(e_q, r_q) = e_a$, where $e_q$, $r_q$, and $e_a$ represent the query entity, query relation, and answer entity, respectively.

Various approaches have been developed to conduct effective and efficient KG reasoning~\cite{ConvE,QuatE,zhang2020autosf}. A primary focus of this research is the generation and encoding of effective reasoning paths. Included methods that learn logical rules for path generation \cite{RLogic, RNNLogic, DRUM}, or employ reinforcement learning to discover paths based on query conditions \cite{MINERVA}. With the advent of Graph Neural Networks (GNNs), recent studies like NBFNet~\cite{zhu2021neural} and RED-GNN~\cite{RED-GNN} iteratively aggregate and encode all reasoning paths of a certain length $\ell$. 
To reduce computational complexity, subsequent approaches introduce path pruning strategies~\cite{zhu2023astarnet, Adaprop}. 
However, most current GNN-based approaches utilize rigid, query-agnostic path encoding and pruning strategies, resulting in two primary limitations:

\begin{figure*}[h]
  \centering
  \begin{subfigure}[b]{0.26\textwidth}
    \centering
    \includegraphics[height=3.2cm, width=0.95\textwidth]{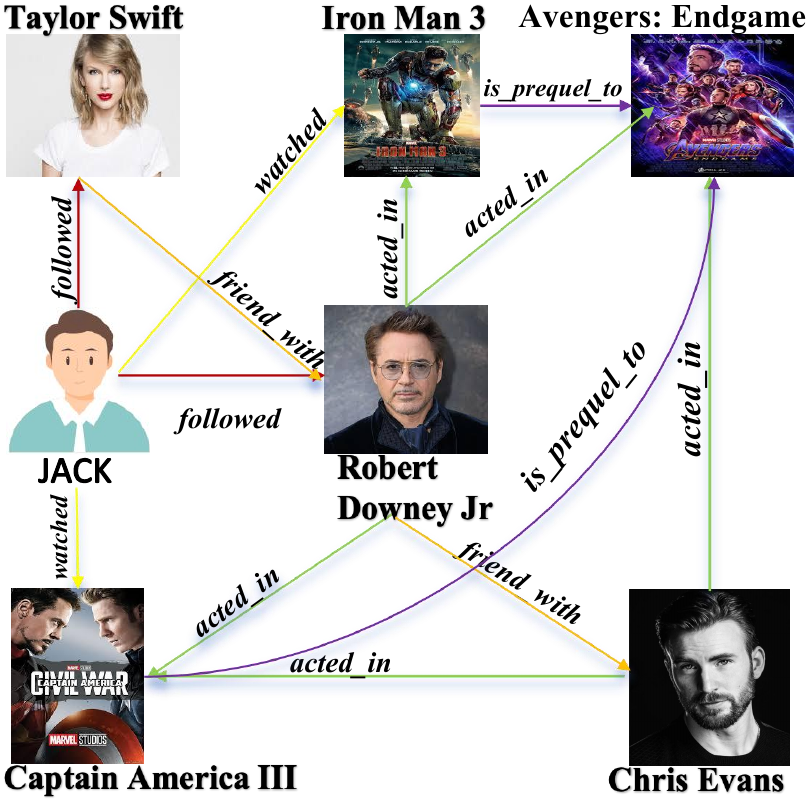}
    \caption{Knowledge Graph}
    \label{fig:1a} 
  \end{subfigure}%
  \begin{subfigure}[b]{0.35\textwidth}
    \centering
    \includegraphics[height=3.2cm, width=0.95\textwidth]{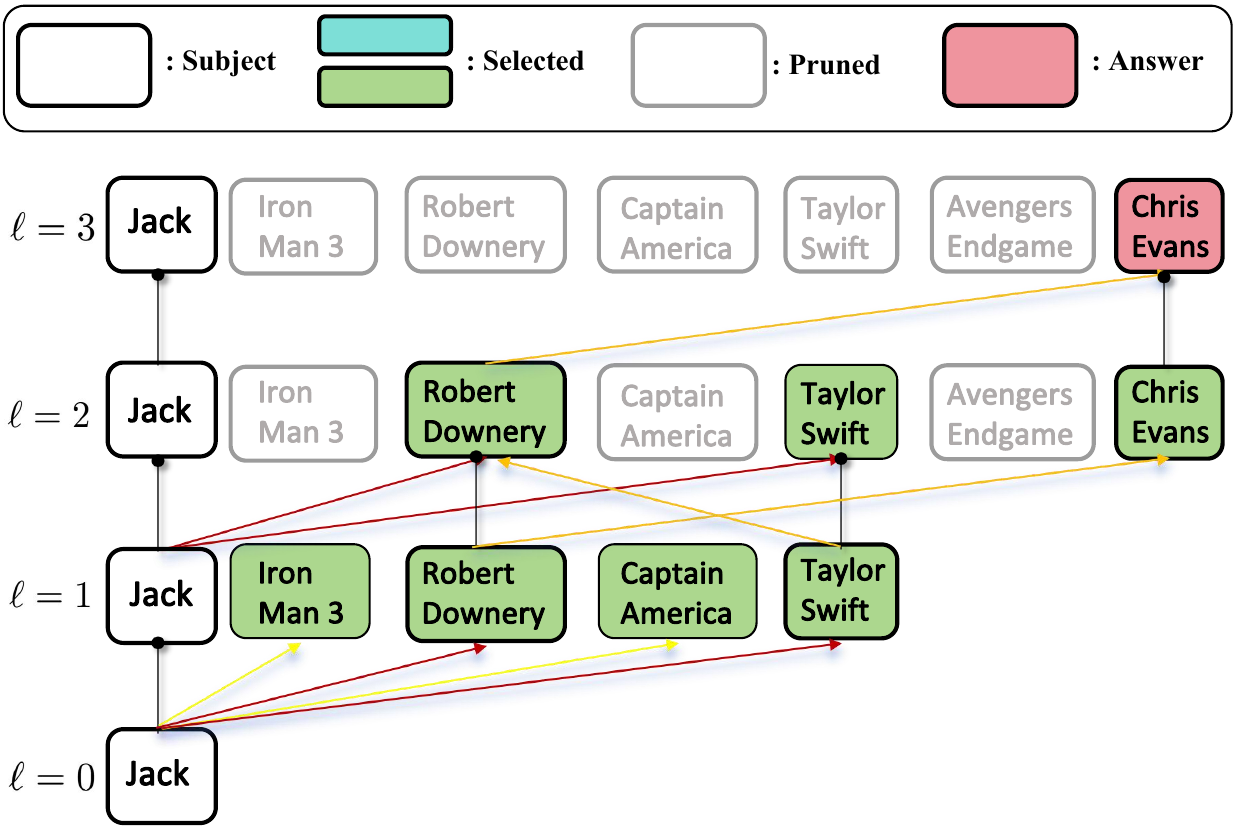}
    \caption{\textit{(Jack, followed, ?)}}
    \label{fig:1b} 
  \end{subfigure}
  \begin{subfigure}[b]{0.35\textwidth}
    \centering
    \includegraphics[height=3.2 cm, width=0.95\textwidth]{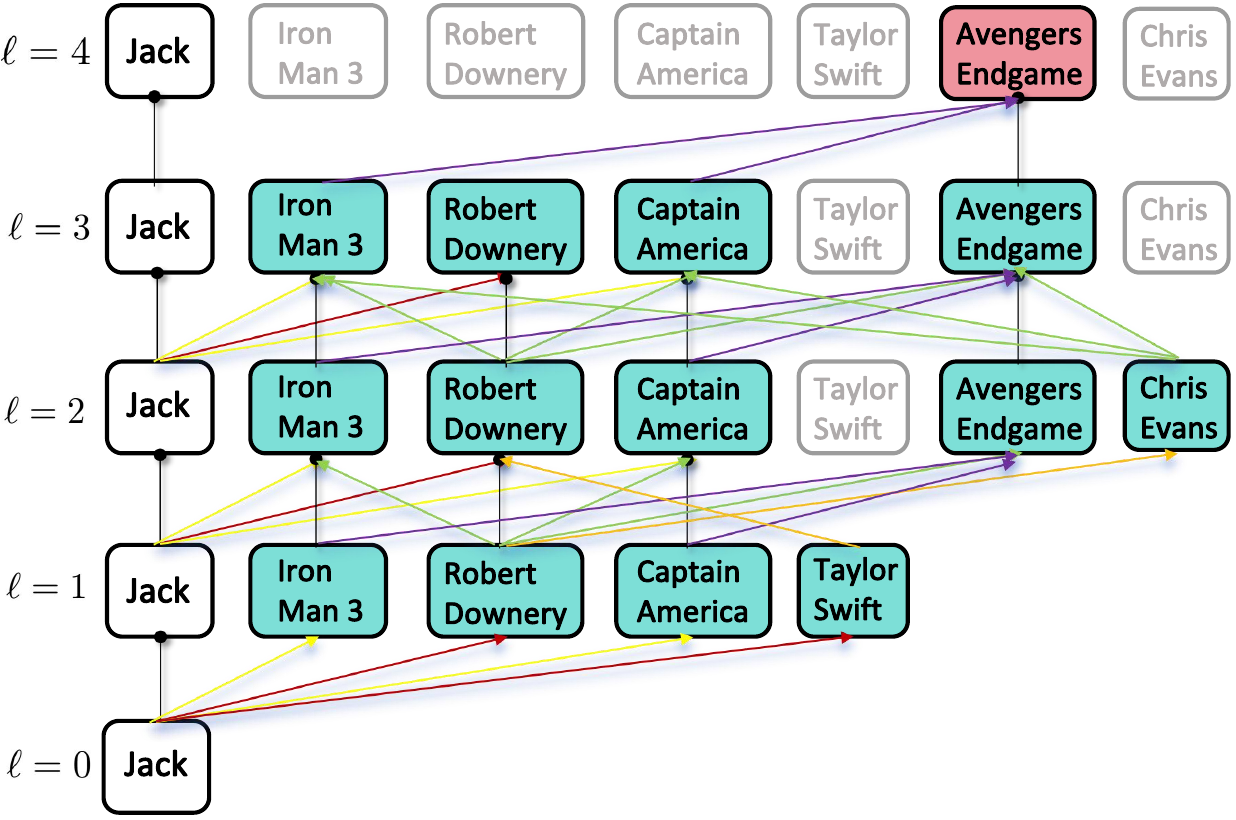}
    \caption{\textit{(Jack, watched, ?)}}
    \label{fig:1c} 
  \end{subfigure}
  \caption{(a) A complex knowledge graph with two queries—\textit{(JACK, followed, ?)} and \textit{(JACK, watched, ?)}—and their respective answers \textit{Chris Evans} and \textit{Avengers: Endgame}. (b) and (c) visualize MoKGR’s personalized path exploration for each query, highlighting its adaptive path length selection and expert-guided pruning, which result in distinct retained paths and entities during reasoning.
}

%
  \label{fig:c&xi}
\end{figure*}

\begin{itemize}[leftmargin=*]
    \item \textbf{Disregarding Dynamic Query Requirements.} 
    Existing methods often employ a fixed hop count for path construction, failing to adapt to the dynamic requirements of individual queries. This uniform approach extends paths to the same depth for every query, overlooking the fact that optimal reasoning paths inherently vary.
    For instance, as illustrated in Fig.~\ref{fig:c&xi}, resolving the query \textit{(JACK, followed, ?)} to find \textit{Chirs Evans} naturally concludes within three hops. Conversely, addressing the query \textit{(JACK, watched, ?)} to identify \textit{Avengers: Endgame} might require exploration beyond three hops to capture critical relationships. 
    These scenarios underscore the need for hop-level personalization, tailoring path exploration to query complexity to improve reasoning efficiency and accuracy.
    
    \item \textbf{Oversimplified Path Exploration Strategies.} 
    Existing methods often rely on overly simplified path exploration strategies, treating all paths equivalently. Although pruning advancements like AdaProp~\cite{Adaprop} and A*Net~\cite{zhu2023astarnet} enhance efficiency, their pruning criteria remain largely uniform, neglecting the distinct significance of individual paths. Effective path exploration should integrate two complementary aspects: (1) \textbf{structural patterns}, capturing entity importance through multi-dimensional assessments to accurately identify high-quality paths; and (2) \textbf{semantic relevance}, assessing the degree of entity-query association, with paths featuring highly relevant entities more likely yielding correct answers. Proper consideration of these aspects can significantly enhance reasoning path quality.
\end{itemize}

To address these limitations, we propose a novel framework called the \textbf{M}ixture \textbf{o}f Length and Pruning Experts for \textbf{K}nowledge \textbf{G}raph \textbf{R}easoning (\textbf{MoKGR}). 
As illustrated in Fig.~\ref{fig:1b} and Fig.~\ref{fig:1c}, MoKGR introduces personalization into path exploration via two complementary innovations. 
First, it employs an \textbf{adaptive length-level selection mechanism}, which functions as a mixture of length experts, dynamically assigning importance weights to various path lengths based on individual queries. This allows shorter paths to be selected when adequate, avoiding unnecessary exploration depth. 
Second, MoKGR utilizes \textbf{specialized pruning experts} that analyze diverse properties: global importance through prediction scores, local structural patterns via attention mechanisms, and semantic relationships through entity-query similarity. Thus, MoKGR comprehensively incorporates structural and semantic considerations into path pruning, ensuring robust and query-specific reasoning paths.
The contributions can be summarized as follows:


\begin{itemize}[leftmargin=*]

\item 
We propose a personalized path exploration strategy for knowledge graph reasoning that adapts to query-specific requirements and entity characteristics, thereby enabling tailor-made reasoning paths without relying on predefined or static relation selection strategies.

\item 
We introduce a novel mixture-of-experts framework that facilitates personalization in knowledge graph reasoning. The incorporating of both adaptive length-level weighting and personalized pruning strategies effectively addresses the critical limitations of fixed path length and uniform path exploration.

\item 
Experimental results on transductive and inductive datasets highlight MoKGR's achievement of superior reasoning accuracy and computational efficiency, enabling it to consistently outperform existing state-of-the-art methods.
\end{itemize}

\section{RELATED WORKS}

\subsection{Path-based Methods for KGs Reasoning}

Path-based reasoning methods aim to construct effective reasoning paths for predicting answer entities through the query function $\mathcal{Q}(e_q, r_q) = e_a$. These methods can be broadly categorized into traditional path reasoning approaches and more recent GNN-based methods.
\paragraph{Traditional Path Reasoning.}
Early path reasoning methods primarily rely on reinforcement learning and rule-based approaches. MINERVA~\cite{MINERVA} pioneered the use of reinforcement learning, training an agent to autonomously traverse the graph from query entity $e_q$ to potential answers. However, these RL-based approaches often face challenges due to the inherent sparsity of KGs. As an alternative direction, rule-based methods focus on learning logical rules for path generation. DRUM~\cite{DRUM} employs bidirectional LSTM to capture sequential patterns and enable end-to-end rule learning, while RLogic~\cite{RLogic} combines deductive reasoning with representation learning through recursive path decomposition. Despite their contributions, these methods typically focus on sequential pattern extraction without considering personalized path exploration requirements.

\paragraph{GNN-based Path Reasoning.}
More recently, GNN-based methods have achieved superior performance by effectively leveraging the rich structural information preserved in graphs. Methods such as NBFNet~\cite{zhu2021neural} and RED-GNN~\cite{RED-GNN} construct reasoning paths by iteratively aggregating information from the $\ell$-length neighborhood of the query entity $e_q$. To enhance path quality, various optimization techniques have been proposed. A*Net~\cite{zhu2023astarnet} and AdaProp~\cite{Adaprop} introduce pruning mechanisms based on priority functions and score-based filtering respectively. One-Shot-Subgraph~\cite{one-shot-subgraph} improves efficiency by utilizing PPR scores for preliminary path exploration and pruning. However, these approaches typically employ rigid path exploration strategies with fixed length distances and uniform pruning criteria, limiting their adaptability to query-specific requirements in practical scenarios.

\subsection{Mixture of Experts}

The Mixture-of-Experts (MoE) paradigm represents a divide-and-conquer learning strategy where multiple specialized expert models collaborate to solve complex tasks, with a gating mechanism dynamically routing inputs to the most suitable experts. The foundational concept of MoE can be traced back to \cite{jordan1994hierarchical} and has been widely adopted in vision~\cite{riquelme2021scaling}, multi-modal learning~\cite{mustafa2022multimodal}, and multi-task learning~\cite{zhu2022uni}.

In graph-related applications, MoE has demonstrated significant advantages by leveraging diverse graph properties. MoKGE~\cite{yu2022diversifying} integrates experts specializing in different subspaces and relational structures of commonsense KGs, achieving diverse outputs in generative commonsense reasoning. MoG~\cite{zhang2023graph} incorporates pruning experts with complementary sparsification strategies, where each expert executes a unique pruning method to customize pruning decisions for individual nodes. Meanwhile, GMoE~\cite{wang2023graph} deploys message-passing experts that specialize in different hop distances and aggregation patterns, enabling nodes to adaptively select suitable experts for information propagation based on their local topologies.

\section{THE PROPOSED METHOD}

\subsection{Preliminary} \label{Preliminary}

As introduced, the reasoning task in KGs is to find the answer entity $e_a$
given a query $(e_q, r_q, ?)$, which we denote as $q = (e_q, r_q)$.
To solve this task,
GNN-based path reasoning methods, such as NBFNet and RED-GNN,
encode all paths up to length $L$ between $e_q$ and $e_a$ 
into a query-specific representation $\bm{h}_{e_a|q}^L \in \mathbb{R}^d$,
and use it to compute the score of candidate entity $e_a$.
The representation $\bm{h}_{e_y|q}^{\ell}$ at iteration $\ell$ is recursively computed via the following message passing function:

\begin{align}
\bm{h}_{e_y|q}^{\ell} = 
\!\!\!\!\!\!\!\!\!\!\!\bigoplus_{(e_x,r,e_y)\in\mathcal N_e(e_y)} \!\!\!\!\!\!\!\!\!\!\!
\big( \bm{h}_{e_x|q}^{\ell-1} \otimes \bm{w}_q^\ell(e_x, r, e_y) \big),
\label{eq:gnn}
\end{align}
where $\bm{h}_{e_x|q}^{\ell-1}$ encodes all paths of length up to $\ell-1$ from $e_q$ to $e_x$, and $\bm{w}_q^\ell(e_x, r, e_y)$ is an edge-specific weight conditioned on the query $q$. 
The operator $\otimes$ combines the path representation with the current edge encoding to form a new path of length $\ell$, and $\bigoplus$ aggregates multiple such paths reaching $e_y$.
We initialize all representations with $\bm{h}_{e_y|q}^{0} = \bm{0}$ for any entity $e_y \in \mathcal{V}$,
and entities $e_y$ that are further than $\ell$ steps from $e_q$ will have $\bm{h}_{e_y|q}^{\ell} = \bm{0}$.
After $L$ iterations of Eq.~\eqref{eq:gnn}, the final score of any entity $e_a \in \mathcal{V}$ is computed by
\begin{equation}
s_L(q, e_a) = (\bm{w}^L)^\top \bm{h}^L_{e_a|q},
\label{eq:score}
\end{equation}
where $\bm{w}^L \in \mathbb{R}^d$ is a learnable scoring vector.
More details of this path encoding process are provided in Appendix~\ref{Path Encoding Process}.

\subsection{Mixture of Length Experts for Adaptive Path Selection} \label{Mixture of Length Experts for Adaptive Path Selection}

Traditional KGs reasoning methods employ fixed-length path exploration strategies, which fail to capture the varying complexity of different queries
and waste computation cost. 
To address this limitation,
we introduce a mixture of length experts that
adaptively selects  paths with different lengths
and 
a layer-wise binary gating function to encourage shorter paths.

\paragraph{Mixture of length experts.} 
For a given query $(e_q,r_q,?)$, 
we presuppose the minimum and maximum path lengths $L_{\min}$ and $L$, 
respectively, and specify the number of selected path length experts as $k_1$ ($< L-L_{\min}$).
Instead of processing all the queries with paths up to length $L$ in Eq.~\eqref{eq:score},
we introduce a mixture of length experts to
score entities with a set of path lengths in intermediate $\ell\in[L_{\min},L]$.

We enable personalized selection of different path lengths.
Denote $\bm c_q$ as the contextual embedding of query $(e_q, r_q, ?)$ 
(details are given in Appendix~\ref{Design details of $c_q$})
and $\bm E_1\in\mathbb R^{(L-L_{\min})\times d}$
as the learnable expert embedding of paths with lengths from $L_{\min}$ to $L$.
Then,
we can measure the compatibility of each path length expert with
\begin{equation}
\resizebox{0.89\linewidth}{!}{$
    \bm Q(\bm c_q) = \bm E_1 \bm c_q+\epsilon\cdot\text{Softplus}(\bm W_n\bm c_q) \in \mathbb R^{L-L_{\min}},
$}
\label{eq:experts score}
\end{equation}
where $\epsilon \sim \mathcal N(0,1)$ is a Gaussian noise works with $\text{Softplus}$~\cite{dugas2001softplus} to encourage diverse expert selection
and
$\bm W_n\in \mathbb R^{(L-L_{\text{min}}) \times d}$ is a trainable parameter that learns noise scores.
Consequently,
we obtain the set $\mathcal A:= \text{Top}_{k_1}(\bm Q(\bm c_q))$ as the indices of selected path lengths.
Then the importance of layer $\ell\in \mathcal A$ can be computed with softmax function

\begin{equation}
    g_q(\ell) = \frac{\exp([\bm Q(\bm c_q)]_{\ell}/\tau)}{\sum_{\ell'\in\mathcal A}\exp([\bm Q(\bm c_q)]_{\ell'}/\tau)}.
\label{eq:gating}
\end{equation}

We  then compute the score of an answer entity $e_a$ with the gated outputs of selected experts with different path lengths
\begin{equation}
\Psi(e_a) = \sum\nolimits_{\ell \in A} g_q(\ell) \cdot s_l(q, e_a) , 
\label{eq:score-moe}
\end{equation}
where the score $s_l(q, e_a) = (\bm w^\ell)^\top \bm h_{e_a|q}^\ell$
at different $\ell$ is defined similarly with Eq.~\eqref{eq:score}.

\paragraph{Layer-wise binary gating function.}  \label{Layer-wise binary gating function}
Even though Eq.~\eqref{eq:score-moe}
can adaptively control the importance of different path length $\ell$,
a limitation still exists that the paths with length from $1$ to $L$ should be explored and encoded.
This can lead to significant computation costs at large layers.
To address this issue, we introduce a layer-wise binary gating function to encourage the model to explore shorter paths. Specifically, during training, we employ a differentiable statistics-based binary gating function $g_b({\ell}) \in (0,1)$ calculated by the Gumbel-Sigmoid~\cite{jang2017categorical}
transformation that evaluates path quality based on layer-wise feature distributions to learn a natural bias towards shorter paths while maintaining differentiability.
(details are given in Appendix~\ref{Design details of gating function}).
We use $g_b({\ell})$ to control the update of the message function in Eq.~\eqref{eq:gnn}
with $\bm h^{\ell}_{e_y|q} \leftarrow g_b({\ell}) \cdot \bm h^\ell_{e_y|q}$.
During inference, we further strengthen this preference through a deterministic truncation strategy,
where $g_b({\ell})=1$, if the paths should continue grow, otherwise $g_b({\ell})=0$.

The iterative message passing process will immediately stop if $g_b({\ell})=0$.
This length control mechanism enables the model to systematically prefer shorter paths when they provide sufficient evidence for reasoning,
improving inference efficiency.

\subsection{Mixture of Pruning Experts for Personalized Path Exploration}

Apart from selecting path length adaptively,
we propose to encourage personalized path pruning,
which incorporates both structural patterns and semantic relevance in KGs reasoning.

We build upon the node-wise pruning mechanism established in AdaProp~\cite{Adaprop}.
Denote $\mathcal V^{\ell}$ as the set of entities that are covered by the message function Eq.~\eqref{eq:gnn} at step $\ell$.
$\mathcal V^{\ell}$ contains all the ending entities of paths with lengths up to $\ell$.
When expanding from $\mathcal V^{\ell-1}$ to $\mathcal V^{\ell}$,
we select Top-$K^{\ell}$ entities from  $\mathcal V^{\ell}$
as an approach to control the number of selected paths.
To implement this fine-grained personalization of path exploration,
we propose three specialized pruning experts 
with different scoring function $\phi_{i}^{\ell}(\cdot)$
that analyzes the importance of entities $e_a\in \mathcal V^{\ell}$
from complementary perspectives.

\begin{itemize}[leftmargin=*]
    \item The Scoring Pruning Expert evaluates the overall contribution to reasoning with the layer-wise score: 
    $\phi_{\text{Sco}}^{\ell}(e_a)=s_l(q,e_a)=(\bm w^\ell)^\top\bm h^\ell_{e_a|q}$.
    \item The Attention Pruning Expert specifically addresses the structural patterns by examining relation combinations and connectivity patterns through an attention mechanism. 
    This expert identifies entities that are important to at least one path connected.
    As defined in Eq.~\eqref{eq:gnn}, at each length $\ell -1$, for each entity $e_y \in \mathcal V^\ell$, we calculate the attention scores $ \alpha$ (detailed computation process is provided in Appendix~\ref{Path Encoding Process}) for all neighboring edges $(e_x,r,e_y) \in \mathcal N_e(e_y)$ where $e_x \in \mathcal V^{\ell -1}$, and assign the maximum attention score among all edges connected to $e_y$ as the attention score of entity $e_y$: $\phi_{\text{Att}}^{\ell}(e_a)=\text{Max}( \alpha(e_x,r,e_a)|(e_x,r,e_a)\in \mathcal N_e(e_a))$.
    \item The Semantic Pruning Expert focuses on semantic relevance by computing the semantic alignment between entities and query relations, 
    ensuring that the selected paths contain thematically coherent concepts that are meaningfully related to the query context. 
    For instance, when reasoning about movie preferences, this expert would favor paths containing entertainment-related entities and relations.  
    We use cosine similarity to measure the coherence:
    $\phi_{\text{Sem}}^{\ell}(e_a)=\cos \left(\bm{h}^\ell_{e_a|q}, \bm w_{r_q}^\ell \right)$.

\end{itemize}

To adaptively combine insights from these path evaluation experts, at each layer $\ell$, similar as the lengths experts defined in Section~\ref{Mixture of Length Experts for Adaptive Path Selection}, denote $\bm c^\ell_v$ as the contextual embedding and $\bm E_2^{\ell} \in \mathbb R^{3 \times d}$ as the learnable embedding of pruning experts at length $\ell$, we can similarly get $\bm Q^\ell(\bm c^\ell_v) \in \mathbb R^{3}$ as defined in Eq.~\eqref{eq:experts score}. Let $\mathcal V_{\phi _i}^{\ell}$ denote the set of entities retained by expert $i$ and $k_2$ denote the predefined number of retained pruning experts, 
the entities retained in the $\ell$-th layer are the union of the entities retained by each selected pruning experts as:
\begin{equation}
\resizebox{0.89\linewidth}{!}{$
    \mathcal V_{\phi}^{\ell} \!=\! \{\cup_ {i\in{\text{TopK}_{k_{2}}}(\bm Q^{\ell}(\bm c^\ell_v))} \mathcal V^{\ell}_{\phi_i}|\mathcal V_{\phi_i}^\ell \!=\! \text{TopK}_{K^\ell}(\phi_i^{\ell}(e_a))\}.
    $}
\end{equation}

To further enhance path quality, we introduce an adaptive path exploration strategy that dynamically controls the exploration breadth. Our strategy allows $K^\ell$ to increase with path exploration depth in early stages, while decreasing at larger depth (Detailed description in Appendix~\ref{Sampling number function design}). This strategy enables thorough exploration of promising path regions while preventing noise accumulation from overextended paths.

\subsection{Training Details}

\paragraph{Task Loss} 

To enable effective personalized path exploration, we formulate a task loss that jointly optimizes the GNN parameters and expert model parameters. 
The task loss is defined as:
\begin{equation}
\resizebox{0.89\linewidth}{!}{$
\mathcal L_\text{task} =\!\!\!\!\!\!\!\!\!\!\!\!\sum\limits_{\tiny(e_q,r_q,e_a) \in \mathcal Q_{tra}} \!\left[-\Psi(e_a) + \log \!\!\sum\limits_{\tiny e_o \in \mathcal V} \!\!\exp(\Psi(e_o))\right],
$}
\label{eq:taskloss}
\end{equation}
where the first part is the total score of the positive triple $(e_q,r_q,e_a)$ in the set of training queries, and the second part contains the total scores of all triple with the same query $(e_q,r_q,?)$.

\paragraph{Experts Balance Loss} \label{sec:experts_balance}
To achieve balanced and effective path exploration and address the potential ``winner-takes-all'' problem \cite{lepikhin2020gshard}, we follow \cite{wang2023graph} by introducing several regularization terms. These terms prevent the model from overly relying on specific exploration strategies or experts (Details are provided in Appendix~\ref{Loss function calculation supplement}). The importance loss is defined as:
\begin{equation}
    \begin{aligned}
        &\text{Importance}(\mathcal C) = \sum_{\bm c \in \mathcal C}\sum_{ g \in \mathcal G(\bm c)} g , \\
        &\mathcal L_{\text{Importance}}(\mathcal C) = \text{CV}(\text{Importance}(\mathcal C))^2 ,
    \end{aligned}
    \label{Eq:importance loss}
\end{equation}
where $g \in \mathcal G(\bm c)$ denotes the output of the experts' gating mechanism as calculated in Eq.~\eqref{eq:gating}, and $\text{CV}(\bm X) = \sigma(\bm X)/\mu(\bm X)$ represents the coefficient of variation of input $\bm X$. This formulation yields the length expert importance loss $\mathcal L_l$ and pruning importance loss $\mathcal L_p$. Furthermore, we introduce a load balancing loss for length experts:

\begin{equation}
    \mathcal L_{\text{load}}= \text{CV}(\sum_{\bm c_q \in \mathcal C}\sum_{ p \in P(\bm c_{q},\ell)}p)^2 ,
\end{equation}
where $p$ represents the node-wise probability in the batch.

The final training objective combines these balance terms with the main reasoning task:
\begin{equation}
    \mathcal L = \mathcal L_\text{task} + \lambda_1 (\mathcal L_l + \mathcal L_p) + \lambda_2 \mathcal L_\text{load} ,
\end{equation}
where $\lambda_1,\lambda_2$ are hand-tuned scaling factors.

\begin{algorithm}[ht]
\caption{MoKGR Algorithm Analysis}
\label{alg:MoKGR Algorithm}
\begin{algorithmic}[1]
\REQUIRE Number of length and pruning experts $k_{1}$ and $k_{2}$, range of paths length $[L_{\min}, L]$.
\ENSURE Optimized GNN model parameters $\Theta$ and expert model parameters $\mathbb W$.

\WHILE{not converged}
    \FOR {Each batch of queries $\{(e_q, r_q, e_a)\}$ from $\mathcal{Q}_{tra}$, $\ell \in [1,L]$}
        \IF{$\ell==L_{\min}$} \label{alg: length}
            \STATE Compute context representation $\bm{c}_q$;
            \STATE From $\ell \in [L_{\min}, L]$ select Top-$k_{1}$ length experts via $\bm{Q}(\bm c_q)$;
        \ENDIF
        \STATE Select Top-$k_{2}$ active pruning experts via $\bm Q^\ell(\bm c^\ell_v)$; \label{alg: select pruning}
        \STATE Union selected experts to get entities $\mathcal{V}_\phi^\ell$; \label{alg: union entities}
        \STATE Update entities in $\mathcal{V}_\phi^\ell$ via Eq.~\eqref{eq:gnn}; \label{alg:update path representation}
        \IF{$\ell$ is the selected length expert}
        \STATE Update the entity scores $\Psi(e_a)$ for $e_a \in  \mathcal{V}_\phi^\ell$; \label{alg:update scores}
        \STATE \textbf{break if}  early stopping condition meet: $g_b(\ell)=0$; \label{alg: early stopping}
        \ENDIF
    \ENDFOR
    
    \STATE Compute total loss $\mathcal{L}$ combining task and balance losses; \label{alg: compute losses}
    \STATE Update $\Theta$ and $\mathbb W$ using gradient of $\mathcal{L}$;
\ENDWHILE
\RETURN Optimized parameters $\Theta$, $\mathbb W$.
\end{algorithmic}
\end{algorithm}

The full algorithm of MoKGR is shown in Algorithm~\ref{alg:MoKGR Algorithm}. 
For each layer's message passing, we first compute the selected pruning experts and their corresponding weights $\bm Q^\ell(\bm c^\ell_v)$ in line~\ref{alg: select pruning}. Then we obtain the final set of preserved entities $\mathcal V_\phi^\ell$ by combining the selected experts in line~\ref{alg: union entities}. Subsequently, we perform message passing only on the preserved entity set in line~\ref{alg:update path representation}. When our message passing reaches layer $L_{\min}$ as shown in line~\ref{alg: length}, we first calculate the selected experts and their corresponding weights $\bm{Q}(\bm c_q)$ through the length expert gating mechanism. In the subsequent layers $\ell \in [L_{\min},L]$, if $\ell$ is a selected length expert, we compute and update the scores of entities selected by pruning experts at that layer in line~\ref{alg:update scores}. If our layer-wise binary gating function is activated, early stopping is performed in line~\ref{alg: early stopping}. After message passing ends or early stopping is triggered, we use the highest-scoring candidate answer entity $e_a$ from all candidates in $\Psi(e_a)$ as the final predicted answer entity.

\section{Experiments}

In this section, we conduct extensive experiments to answer the following research questions: 
(\textbf{RQ1}~\ref{RQ1}) How effective is MoKGR in improving reasoning performance and efficiency compared to existing methods? (\textbf{RQ2}~\ref{RQ2}) How does our expert selection mechanism perform? (\textbf{RQ3}~\ref{RQ3}) How does MoKGR achieve personalized path exploration in practice? (\textbf{RQ4}~\ref{RQ4}) How do different components and hyperparameters affect the model's performance?

\subsection{Experimental Setup} \label{Experiment Setup}

We compares MoKGR with general KGs reasoning methods in both transductive and inductive settings. (The other implementation details and inductive setting  is given in Appendix~\ref{Experimental Detai and Supplementary Results}.)
We use filtered ranking-based metrics for evaluation, namely mean reciprocal ranking (MRR) and Hit@k. Higher values for these metrics indicate better performance. 

\paragraph{Datasets.}
We use six widely used KGs reasoning benchmarks: Family \cite{kok2007statistical}, UMLS \cite{kok2007statistical}, WN18RR \cite{dettmers2017convolutional}, FB15k-237 \cite{toutanova2015observed}, NELL-995 \cite{xiong2017deeppath}, and YAGO3-10 \cite{suchanek2007yago}. 

\paragraph{Baselines.}
We compare the proposed MoKGR with (i) non-GNN methods: ConvE \cite{ConvE}, QuatE \cite{QuatE}, RotatE \cite{RotatE}, MINERVA \cite{MINERVA}, DRUM \cite{DRUM}, AnyBURL~\cite{meilicke2020reinforced}, RNNLogic \cite{RNNLogic}, RLogic \cite{RLogic}, DuASE~\cite{li2024duase} and GraphRulRL~\cite{mai2025graphrulrl}; and (ii) GNN-based methods CompGCN \cite{CompGCN}, NBFNet \cite{zhu2021neural}, RED-GNN \cite{RED-GNN}, A*Net \cite{zhu2023astarnet}, Adaprop \cite{Adaprop}, ULTRA~\cite{galkin2024foundation} and one-shot-subgraph \cite{one-shot-subgraph}.

\subsection{Overall Performance (RQ1)} \label{RQ1}

\begin{table*}[ht]
\centering
\resizebox{\textwidth}{!}{
\setlength\tabcolsep{2.5pt}
\begin{tabular}{lc|ccc|ccc|ccc|ccc|ccc|ccc}
\hline
\multirow{2}{*}{\textbf{Type}} & \multirow{2}{*}{\textbf{Model}} & \multicolumn{3}{c|}{Family} & \multicolumn{3}{c|}{UMLS} & \multicolumn{3}{c|}{WN18RR} & \multicolumn{3}{c|}{FB15k237} & \multicolumn{3}{c|}{NELL-995} & \multicolumn{3}{c}{YAGO3-10} \\
    &        & MRR & H@1 & H@10 & MRR & H@1 & H@10 & MRR & H@1 & H@10 & MRR & H@1 & H@10 & MRR & H@1 & H@10 & MRR & H@1 & H@10 \\
\hline
 \multirow{10}{*}{\textbf{Non-GNN}}  &  ConvE   & 0.912 & 83.7 & 98.2 & 0.937 & 92.2 & 96.7 & 0.427 & 39.2 & 49.8 & 0.325 & 23.7 & 50.1 & 0.511 & 44.6 & 61.9 & 0.520 & 45.0 & 66.0 \\

 &   QuatE     & 0.941 & 89.6 & 99.1 & 0.944 & 90.5 & 99.3 & 0.480 & 44.0 & 55.1 & 0.350 & 25.6 & 53.8 & 0.533 & 46.6 & 64.3 & 0.379 & 30.1 & 53.4 \\
 &    RotatE    & 0.921 & 86.6 & 98.8 & 0.925 & 86.3 & 99.3 & 0.477 & 42.8 & 57.1 & 0.337 & 24.1 & 53.3 & 0.508 & 44.8 & 60.8 & 0.495 & 40.2 & 67.0 \\
 & MINERVA & 0.885 & 82.5 & 96.1 & 0.825 & 72.8 & 96.8 & 0.448 & 41.3 & 51.3 & 0.293 & 21.7 & 45.6 & 0.513 & 41.3 & 63.7 & - & - & - \\
 &    DRUM     & 0.934 & 88.1 & 99.6 & 0.813 & 67.4 & 97.6 & 0.486 & 42.5 & 58.6 & 0.343 & 25.5 & 51.6 & 0.532 & 46.0 & 66.2 & 0.531 & 45.3 & 67.6 \\
  &    AnyBURL     & 0.861 & 87.4 & 89.2 & 0.828 & 68.9 & 95.8 & 0.471 & 44.1 & 55.2 & 0.301 & 20.9 & 47.3 & 0.398 & 27.6 & 45.4 & 0.542 & 47.7 & 67.3 \\
  &  RNNLogic    & 0.881 & 85.7 & 90.7 & 0.842 & 77.2 & 96.5 & 0.483 & 44.6 & 55.8 & 0.344 & 25.2 & 53.0 & 0.416 & 36.3 & 47.8 & 0.554 & 50.9 & 62.2 \\
  &    RLogic    & - & - & - & - & - & - & 0.477 & 44.3 & 53.7 & 0.310 & 20.3 & 50.1 & 0.416 & 25.2 & 50.4 & 0.36 & 25.2 & 50.4 \\
    &    DuASE    & 0.861 & 81.2 & 90.8 & 0.845 & 72.5 & 85.5 & 0.489 & 44.8 & 56.9 & 0.329 & 23.5 & 51.9  & 0.423 & 37.2 & 59.2 & 0.473 & 38.7 & 62.8 \\
    &    GraphRulRL    & 0.928 & 87.4 & 95.1 & 0.869 & 84.5 & 97.1 & 0.483 & 44.6 & 54.1 & 0.385 & 31.4 & 57.5 & 0.425 & 27.8 & 52.7 & 0.432 & 35.4 & 51.7 \\
  \cline{1-20}
 \multirow{7}{*}{\textbf{GNNs}} &   CompGCN    & 0.933 & 88.3 & 99.1 & 0.927 & 86.7 & 99.4 & 0.479 & 44.3 & 54.6 & 0.355 & 26.4 & 53.5 & 0.463 & 38.3 & 59.6 & 0.421 & 39.2 & 57.7 \\
  &   NBFNet  & 0.989 & \underline{98.8} & 98.9 & 0.948 & 92.0 & \underline{99.5} & 0.551 & 49.7 & 66.6 & 0.415 & 32.1 & \underline{59.9} & 0.525 & 45.1 & 63.9 & 0.550 & 47.9 & 68.6 \\
  &  RED-GNN     & \underline{0.992} & \underline{98.8} & \textbf{99.7} & 0.964 & 94.6 & 99.0 & 0.533 & 48.5 & 62.4 & 0.374 & 28.3 & 55.8 & 0.543 & 47.6 & 65.1 & 0.559 & 48.3 & 68.9 \\
  &  A*Net     & 0.987 & 98.4 & 98.7 & 0.967 & 94.8 & 99.1 & 0.549 & 49.5 & 65.9 & 0.411 & 32.1 & 58.6 & 0.549 & 48.6 & 65.2 & 0.563 & 49.8 & 68.6 \\
  &   AdaProp    & 0.988 & 98.6 & 99.0 & 0.969 & \underline{95.6} & \underline{99.5} & 0.562 & 49.9 & \underline{67.1} & 0.\underline{417} & 33.1 & 58.5 & \underline{0.554} & \underline{49.3} & 65.5 & 0.573 & 51.0 & 68.5 \\
   &    ULTRA   & 0.913 & 86.6 & 97.2 & 0.915 & 89.6 & 98.4 & 0.480 & 47.9 & 61.4 & 0.368 & \underline{33.9} & 56.4 & 0.509 & 46.2 & \underline{66.0} & 0.557 & 53.1 & 71.0 \\
  & one-shot-subgraph & 0.988 & 98.7 & 99.0 & \underline{0.972} & 95.5 & 99.4 & \underline{0.567} & \underline{51.4} & 66.6 & 0.304 & 22.3 & 45.4 & 0.547 & 48.5 & 65.1 & \underline{0.606} & \underline{54.0} & \underline{72.1} \\
  \cline{2-20}
  &  \textbf{MoKGR}   & \textbf{0.993} & \textbf{99.1} & \underline{99.3} & \textbf{0.978} & \textbf{96.5} & \textbf{99.6 }& \textbf{0.611} & \textbf{53.9} & \textbf{70.2} & \textbf{0.443} & \textbf{36.8} & \textbf{60.7} &\textbf{0.584} & \textbf{50.7} & \textbf{67.9} & \textbf{0.657} &\textbf{57.7}&\textbf{75.8} \\
\hline
\end{tabular}
}
\caption{Comparison of MoKGR with other KG reasoning methods in the transductive setting. Best performance is indicated by the \textbf{bold} face numbers, and the \underline{underline} means the second best. ``-" means unavailable results. ``H@1" and ``H@10" are short for Hit@1 and Hit@10 (in percentage), respectively.}
\label{tab:transductive}
\end{table*}


\begin{figure}[ht]
  \centering
    \centering
    \includegraphics[width=0.49\textwidth]{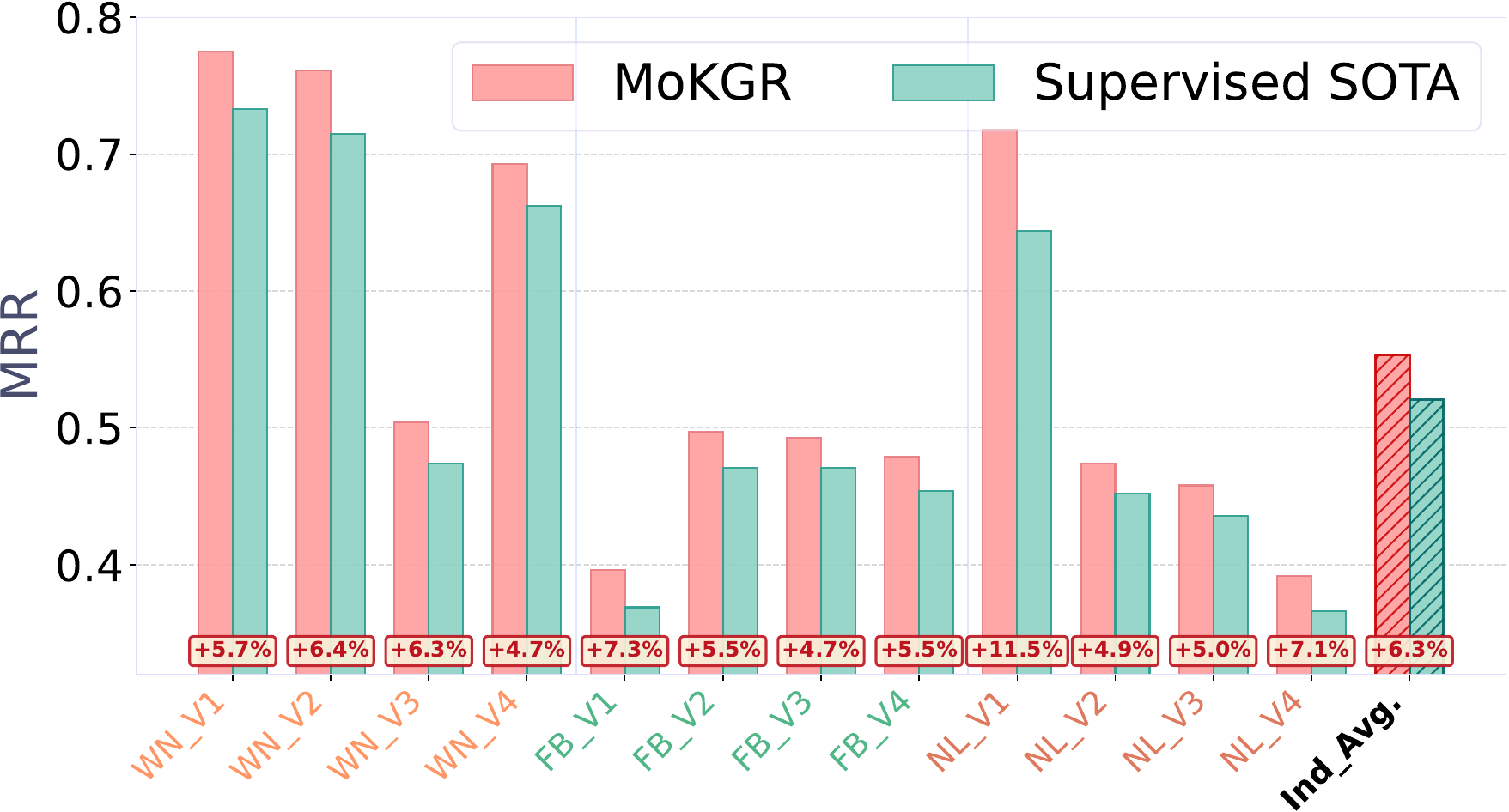}
    \label{fig:inductive}
  \caption{Comparison of MoKGR with supervised state-of-the-art baselines under the inductive setting.}
  \label{fig:main_compare}
\end{figure}
\paragraph{Results.}
Tab.~\ref{tab:transductive} and Fig.~\ref{fig:main_compare} show that our proposed MoKGR exhibits exceptional performance across all benchmark datasets in both transductive and inductive reasoning (Detail implementation and results for inductive setting are given in Appendix~\ref{implementation details for inductive setting}). The experimental results validate several key advantages of our approach. First, the adoption of GNN-based message passing proves to be more effective than non-GNN methods for KGs reasoning, as evidenced by consistent performance improvements across all metrics. Furthermore, compared to 
full-exploration \cite{RED-GNN,zhu2021neural} and fixed pruning strategies \cite{zhu2023astarnet,Adaprop,one-shot-subgraph}, 
our 
MoE system achieves personalization in both pruning and reasoning path length, while implementing a mixture and personalized approach that significantly enhances reasoning accuracy. Notably, our method performs particularly well on the largest dataset YAGO3-10, solving the out-of-memory problem of full-exploration methods, and greatly improving the accuracy compared with other pruning methods.

\begin{figure}[t]
  \centering
  \begin{subfigure}[c]{0.26\textwidth}
    \centering
    \includegraphics[height=2.6cm, width=\textwidth]{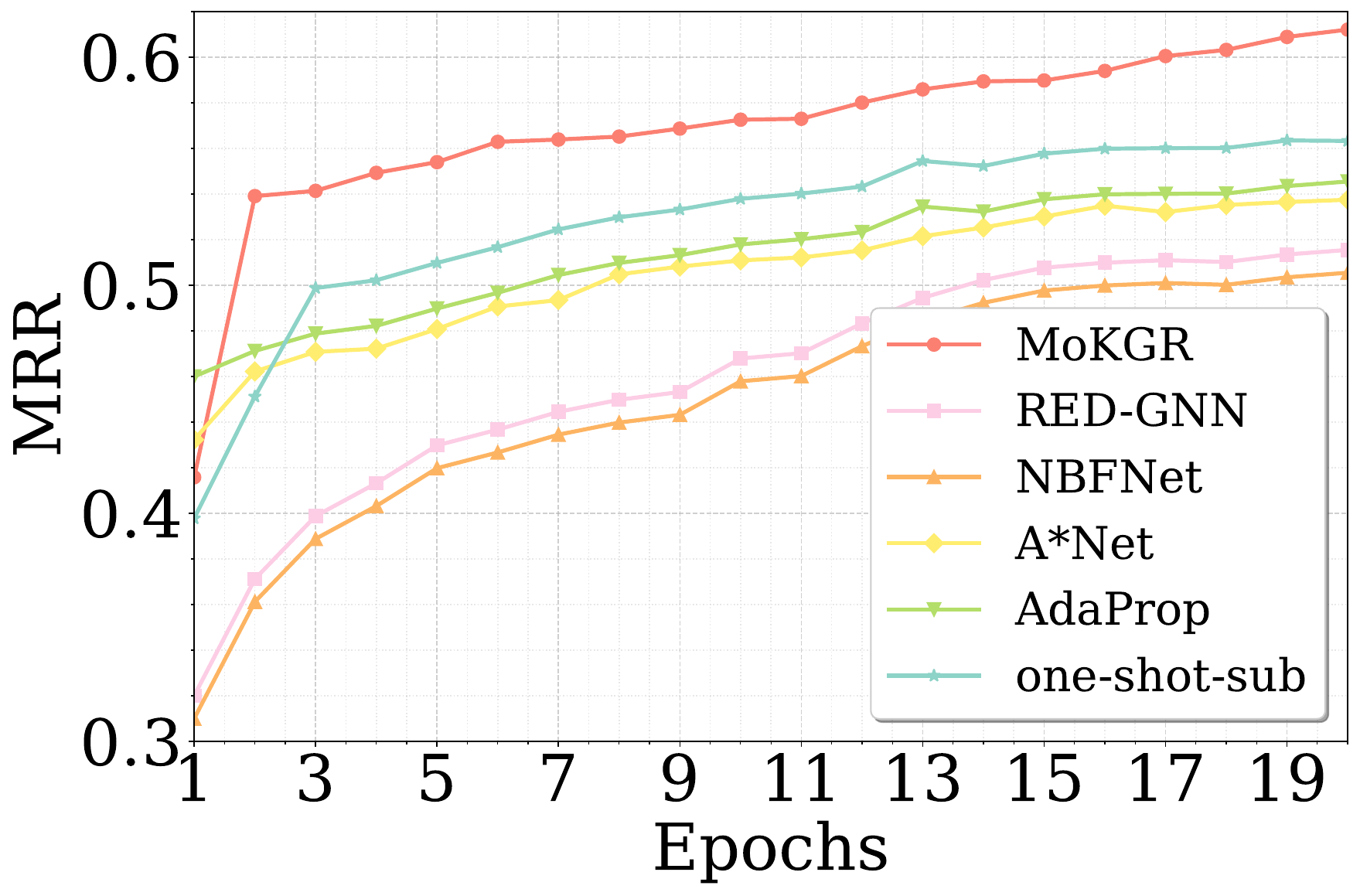}
    \caption{MRR over epochs}
    \label{fig:Comparison of MRR} 
  \end{subfigure}%
  \hfill
  \begin{subfigure}[c]{0.22\textwidth}
    \centering
    \begin{minipage}[c][3cm][c]{\linewidth}
      \centering
      \resizebox{\linewidth}{!}{ 
      \begin{tabular}{c|cc}
          \toprule
          \textbf{Min/Epoch} &\textbf{Training}  &\textbf{Inference}\\
          \midrule
          \textbf{MoKGR} & 111.73 & \textbf{58.3}\\
          \midrule
          RED-GNN &  1382.9  & 802.2\\  
          NBFNet &  493.8& 291.4\\
          A*Net &  112.3  & 92.4 \\
          Adaprop &  \textbf{108.6} & 84.2  \\
          one-shot-subgraph & 147.9  & 71.9 \\
          \bottomrule
      \end{tabular}
      }
    \end{minipage}
    \caption{Average time per epoch}
    \label{Tab:Comparison of Effeciency}
  \end{subfigure}%
  \caption{Comparison between MoKGR and current state-of-the-art methods in YAGO3-10 dataset.}
  \label{fig:result_comparation}
\end{figure}

\paragraph{Learning Process Comparison.}
To comprehensively evaluate the effectiveness of MoKGR, we analyze a few SOTA methods on the YAGO3-10 dataset. 
We tracked both the MRR performance and computational time (training/inference) over 20 epochs. As shown in Fig.~\ref{fig:result_comparation}, MoKGR exhibits two distinctive advantages: 1) rapid convergence during early training phases, particularly evident in the steep curves within the first five epochs, and 2) stable performance growth throughout the training process. In contrast, other methods show slower convergence and more erratic progression, particularly in later epochs. 
The computational efficiency analysis presented in Table~\ref{Tab:Comparison of Effeciency} demonstrates that 
MoKGR achieves significantly faster inference times compared to other approaches, while its training time is substantially lower than full-exploration methods and comparable to other pruning approaches.

\subsection{Expert Selection Analysis (RQ2)} \label{RQ2}

\begin{figure}[h]
  \centering
  \begin{subfigure}[b]{0.23\textwidth}
    \centering
    \includegraphics[width=\textwidth]{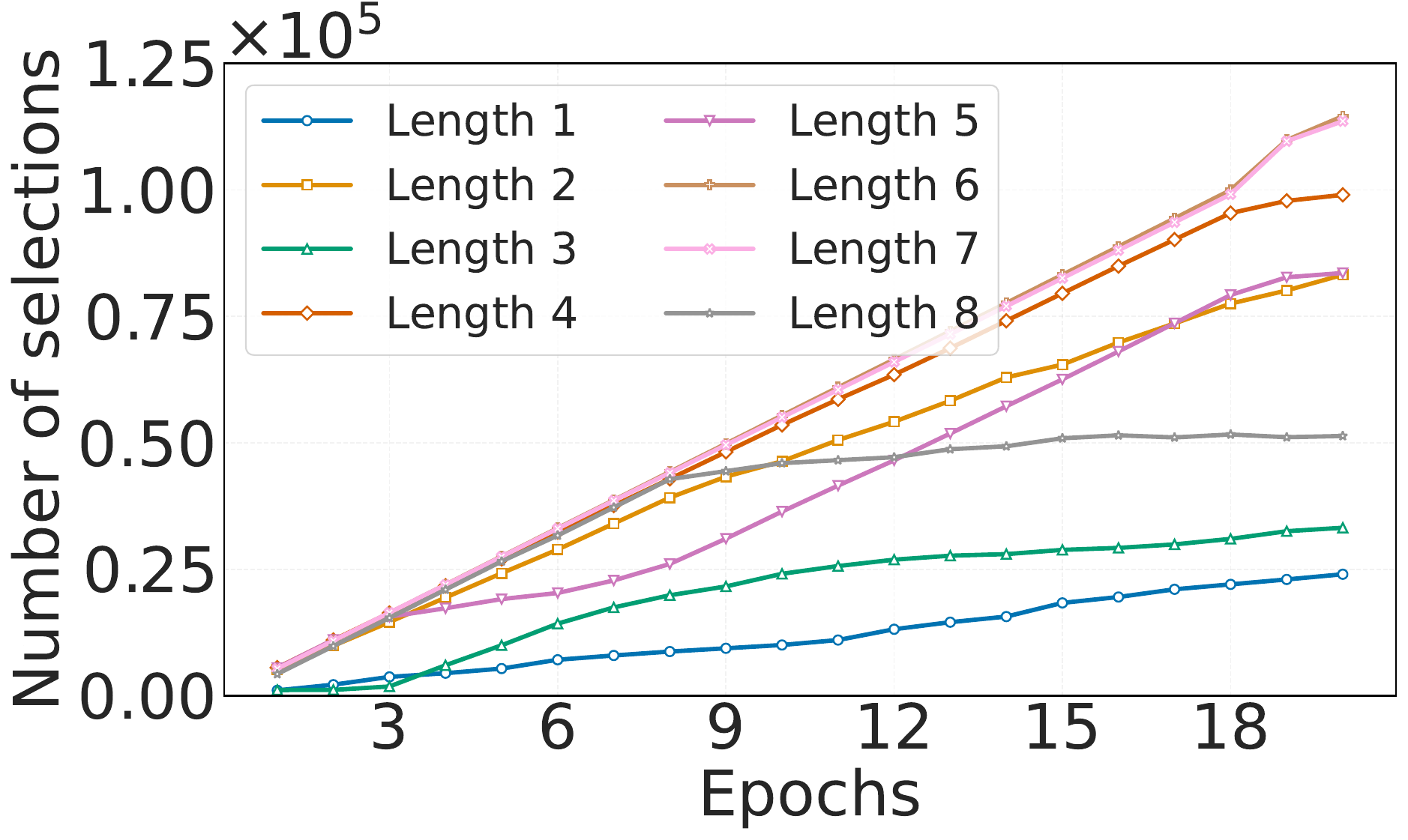}
    \caption{Length experts selection}
    \label{fig:expert selectiona}
  \end{subfigure}%
  \hfill
  \begin{subfigure}[b]{0.23\textwidth}
    \centering
    \includegraphics[width=\textwidth]{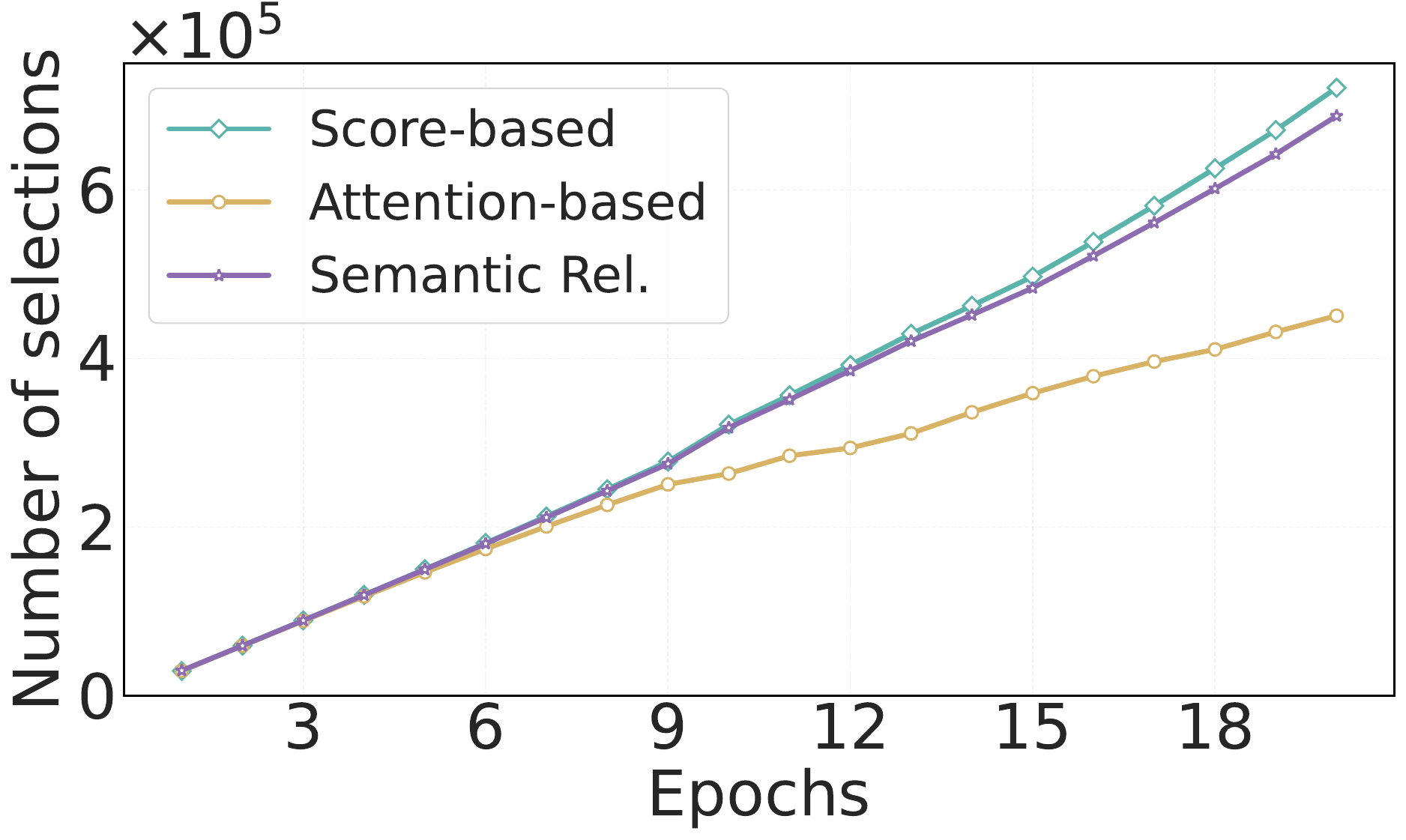}
    \caption{Pruning experts selection}
    \label{fig:expert selectionb} 
  \end{subfigure}%
  \caption{Selection Curves of Two Experts.}
  \label{fig:expert selection}
\end{figure}

\paragraph{Selection Patterns of Length and Pruning Experts.}
To validate the necessity and effectiveness of our multi-expert approach, we conducted comprehensive experiments analyzing the selection patterns of both length experts and pruning experts on the YAGO3-10 dataset. The length expert selection patterns in Fig.~\ref{fig:expert selectiona} reveal 
that the model demonstrates a preference for mixing medium-length paths.
For paths of length 8, we observe a turning point at Epoch 8, which we attribute to the effect of the gating function described in Section~\ref{Layer-wise binary gating function}. This function suppresses the utilization of longer paths, causing the model to favor shorter path lengths. The expert selection pattern in pruning, illustrated in Fig.~\ref{fig:expert selectionb}, demonstrates the complementarity of expert selections. Despite variations in their selection frequencies, the overall trend maintains a steady upward progression.

\begin{figure}[h]
  \centering
  \begin{subfigure}[b]{0.24\textwidth}
    \centering
    \includegraphics[width=\textwidth]{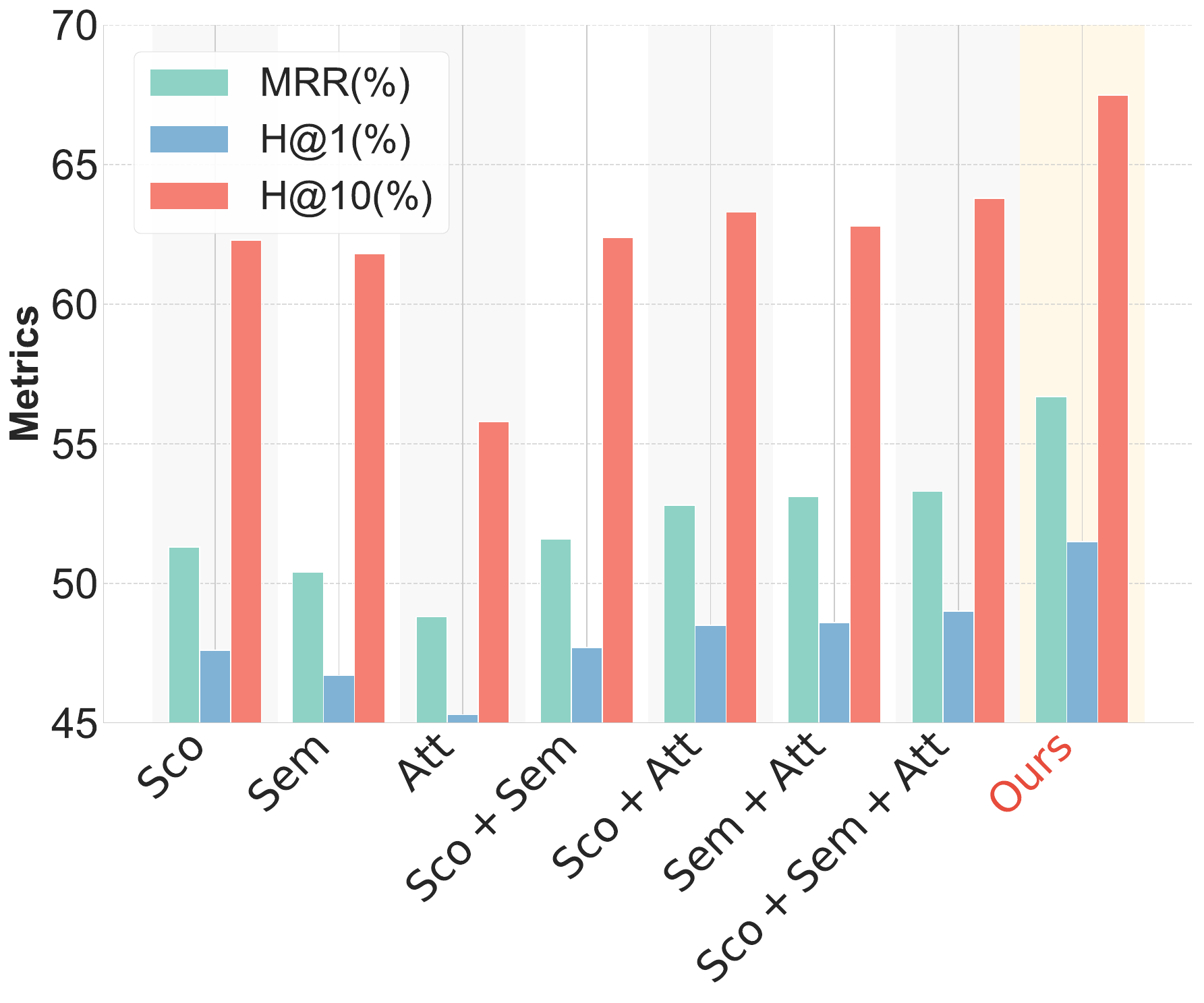}
    \caption{WN18RR}
    \label{Pruning Strategy Selection on WN18RR} 
  \end{subfigure}%
  \begin{subfigure}[b]{0.24\textwidth}
    \centering
    \includegraphics[width=\textwidth]{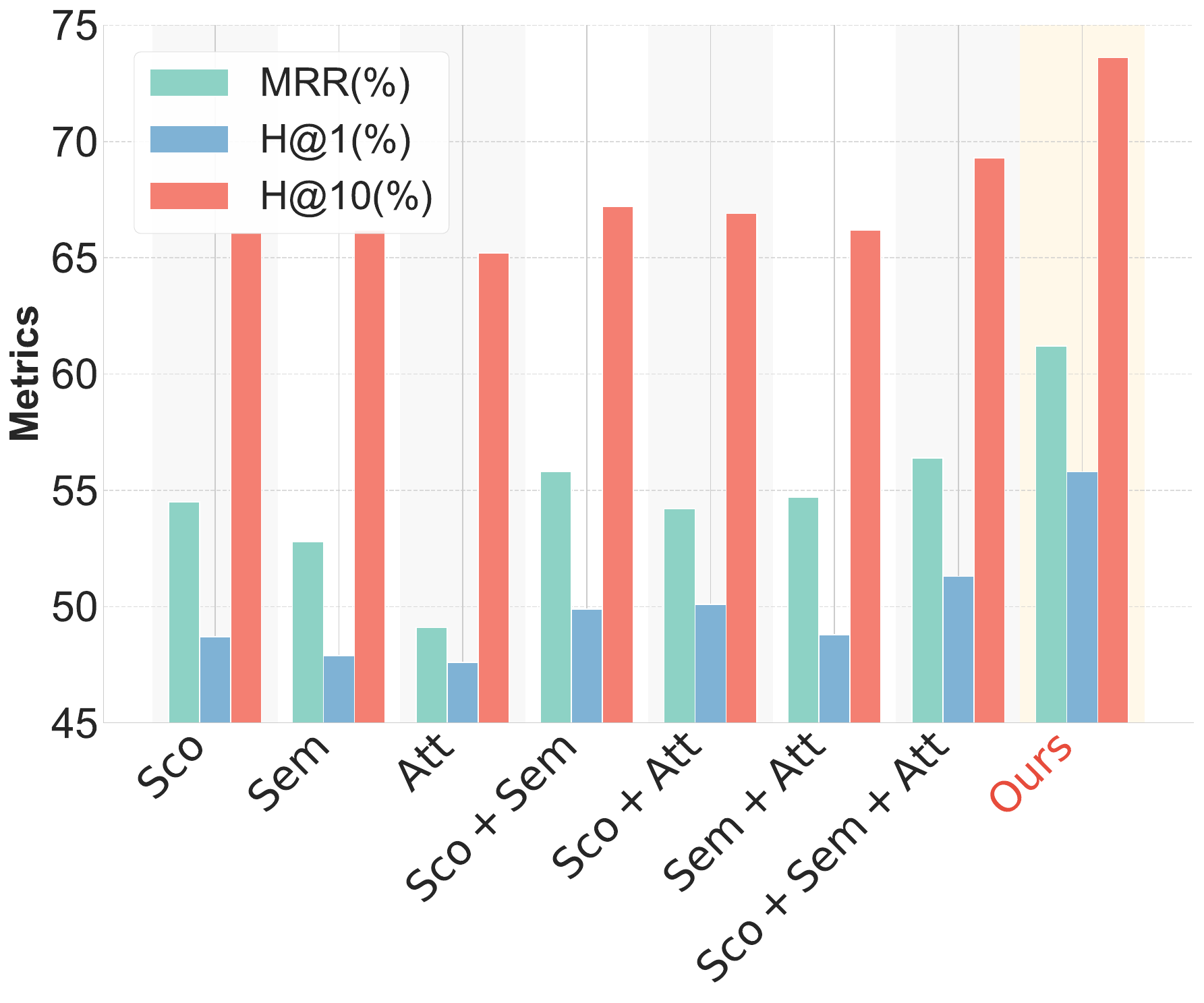}
    \caption{YAGO3-10}
    \label{Pruning Strategy Selection on YAGO} 
  \end{subfigure}%
  \caption{Comparison of Pruning Selection Strategy.}
  \label{fig:Comparison of Pruning Strategy}
\end{figure}

\paragraph{Effectiveness of Adaptive Expert Weighting.}
We evaluate the effectiveness of adaptive weighting by comparing three approaches on WN18RR and YAGO3-10 datasets: single experts (weight 1.0), fixed-weight expert combinations (equal weights), and our adaptive weighting mechanism. Our experiments compare all possible expert combinations ($C^1_3 + C^2_3 + C^3_3$). The results shown in Fig.~\ref{fig:Comparison of Pruning Strategy} demonstrate that compared to single experts, mixed experts perform better, reflecting the complementary nature of different pruning experts. Moreover, compared to fixed expert weights, our dynamic expert weighting shows superior performance, highlighting the necessity of weight personalization for pruning experts. These results validate that adaptive expert weighting is essential for optimal KGs reasoning, as it enables dynamic adjustment of pruning strategies based on query-specific requirements.

\subsection{Case Study (RQ3)} \label{RQ3}

\begin{figure}[h]
  \centering
  \begin{subfigure}[b]{0.23\textwidth}
    \centering
    \includegraphics[width=\textwidth]{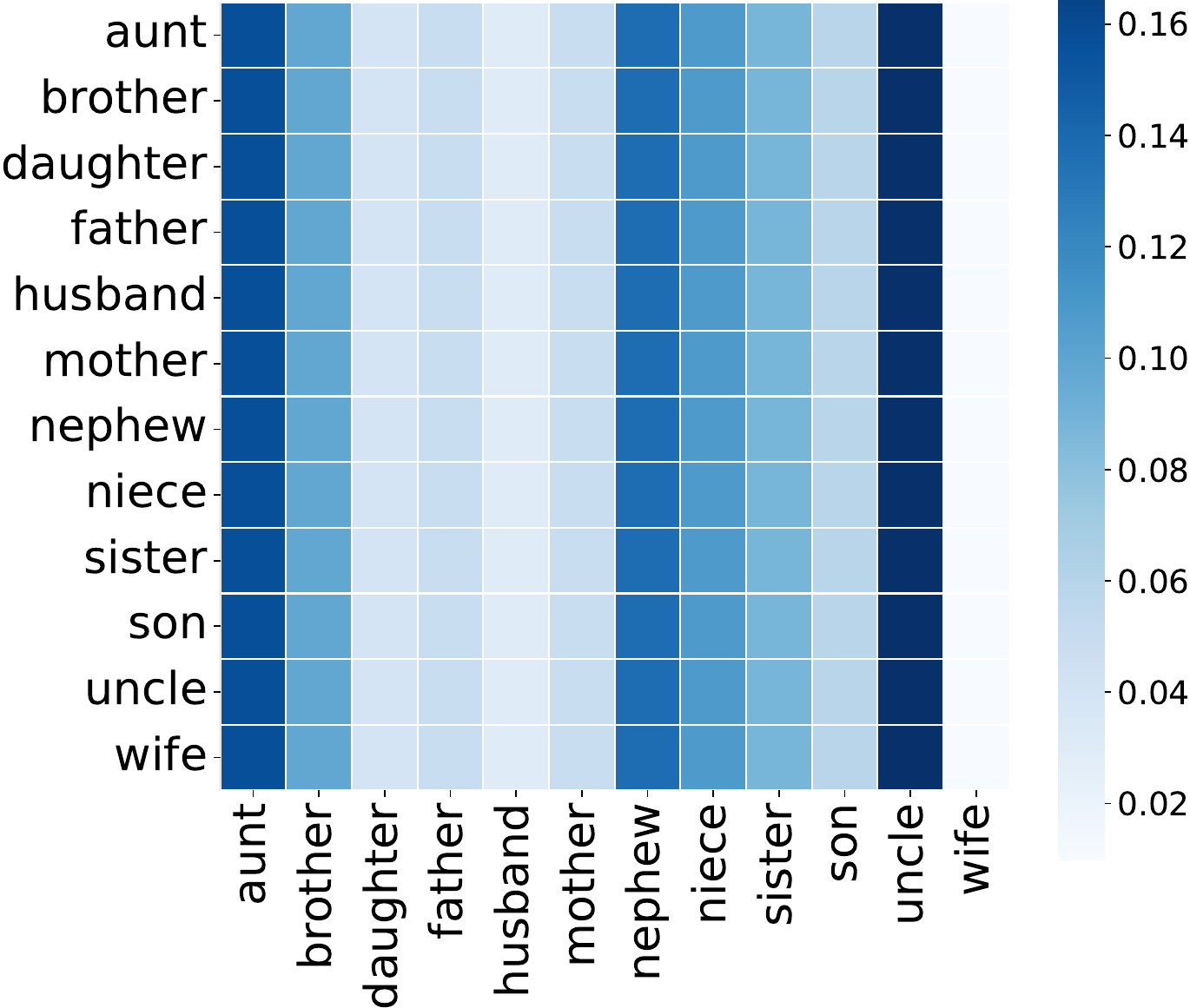}
    \caption{RED-GNN}
    \label{fig:case_a} 
  \end{subfigure}%
  \begin{subfigure}[b]{0.23\textwidth}
    \centering
    \includegraphics[width=\textwidth]{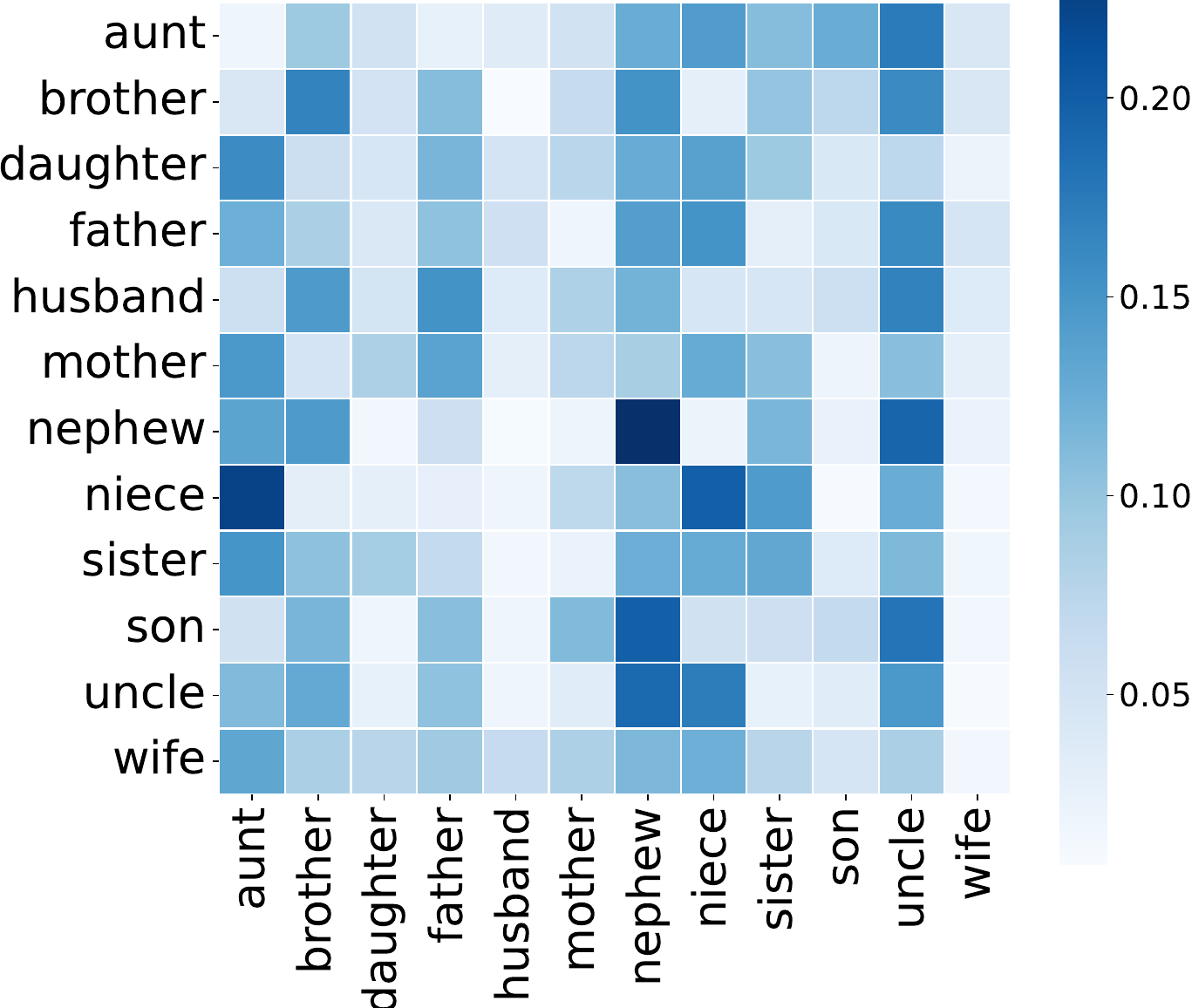}
    \caption{MoKGR}
    \label{fig:case_b}
  \end{subfigure}%
  \caption{Heatmaps of relation type ratios in the reasoning path on Family dataset. Rows represent different query relations, and columns correspond to relation types in the selected edges.}
  \label{fig:case_study}
\end{figure}

To validate MoKGR’s personalized path exploration capabilities, we analyze the reasoning paths selected by the full-propagation method RED-GNN (which passes messages to all neighboring nodes at each propagation step) and compare them with those selected by MoKGR. 
In detail, for each relation $r_i$, we track its occurrence count $N_{{r}_i}^{\ell}$ at each path length $\ell$, and calculate its aggregated importance $N_{{r}i}$ by combining these counts with length experts' weights: $N_{{r}_i} = \sum_{\ell=L_{\min}}^{L} N_{{r}_i}^{\ell} \times g_q({\ell})$. 
The resulting normalized heatmaps reveal MoKGR's query-adaptive behavior. While RED-GNN show uniform distribution (Fig.~\ref{fig:case_a}), MoKGR (Fig.~\ref{fig:case_b}) identifies semantically relevant relations, such as emphasizing \textit{aunt} for \textit{niece} queries and \textit{brother}, \textit{nephew}, and \textit{uncle} for \textit{brother} queries. These case studies confirm that MoKGR effectively adapts its path selection to query semantics, highlighting relevant relations while avoiding redundant exploration, which behavior contrasts sharply with full-propagation methods. Further experiments for this analysis are provided in Appendix~\ref{Supplementary Case Study}.

\subsection{Ablation Study (RQ4)} \label{RQ4}

\begin{table}[htbp]
\centering
\resizebox{0.48\textwidth}{!}{
\begin{tabular}{cc|ccc|ccc}
    \hline
    \multicolumn{2}{c|}{\textbf{Dataset}} & \multicolumn{3}{c|}{\textbf{WN18RR}} & \multicolumn{3}{c}{\textbf{YAGO3-10}}  \\ 
    \hline
    \multicolumn{2}{c|}{\textbf{Metrics}} & \textbf{MRR} & \textbf{H@1} & \textbf{H@10} & \textbf{MRR} & \textbf{H@1} & \textbf{H@10} \\ 
    \hline
    \multicolumn{2}{c|}{\textbf{Ours}}& \textbf{0.611} & \textbf{53.9} & \textbf{70.2} & \textbf{0.657} & \textbf{57.7} & \textbf{75.8}  \\
    \hline
    \multirow{2}{*}{Balancing term}& $\lambda_1=0$ & 0.558 & 50.1 & 66.9 & 0.582 & 53.6 & 71.7 \\
     & $\lambda_2=0$ & 0.560 & 50.6 & 67.0 & 0.596 & 54.1 & 72.4 \\
    \hline
     \multirow{2}{*}{Noise term} &$\epsilon=0$ & 0.560 & 50.8 & 67.1 & 0.586 & 53.7 & 68.3 \\
    &$\epsilon = 0.2$ & 0.561 & 50.8 & 67.2 & 0.602 & 55.1 & 70.3 \\
    \hline
    \multirow{4}{*}{Expert} &$\phi_{\text{Sco}^\diamond}$ & 0.556 & 50.3 & 66.7 & 0.555 & 49.7 & 67.9 \\
    &$\phi_{\text{Att}^\diamond}$  & 0.549 & 49.9 & 66.4 & 0.501 & 48.6 & 66.2 \\
    &$\phi_{\text{Sem}^\diamond}$ & 0.553 & 50.2 & 66.5 & 0.538 & 48.9 & 67.2 \\
    &$\phi _{L^*}$ & 0.552 & 50.5 & 66.5 & 0.569 & 51.0 & 68.4 \\
    \hline
\end{tabular}
}
\caption{Ablation study on WN18RR and YAGO3-10 datasets. $\phi_{(.)^\diamond}$ mean only use this pruning expert (remain original length experts) and $\phi _{L^*}$ means only use L-th layer length expert (remain original pruning experts).}
\label{table:ablation}
\end{table}

To evaluate the contribution of each component in the MoKGR framework, we conducted a comprehensive ablation study on the WN18RR and YAGO3-10 datasets, as shown in Table~\ref{table:ablation}. The results reveal several key findings: 
First, removing any balancing term ($\lambda_1$ or $\lambda_2$) leads to decreased performance, highlighting the importance of expert coordination. 
Furthermore, for the noise parameter $\epsilon$ given in Eq.~\eqref{eq:experts score}, dynamic sampling ($\epsilon \sim \mathcal N(0,1)$) consistently outperforms both no noise ($\epsilon=0$) and fixed noise ($\epsilon=0.2$) settings, demonstrating the benefits of incorporating controlled randomness in expert selection.
Finally, whether in pruning strategies or path length strategies, restricting the model to one fixed strategy significantly weakens its reasoning capability.
    
\section{Conclusion}

In this paper, we introduced \textbf{MoKGR}, a novel framework that 
advances KGs reasoning through personalized path exploration. Our approach addresses two critical challenges in KGs reasoning: adaptive path length selection and comprehensive path evaluation. The adaptive length selection mechanism dynamically adjusts path exploration depths based on query complexity, while the mixture-of-pruning experts framework incorporates structural patterns, semantic relevance, and global importance to evaluate path quality. Through extensive experiments across diverse benchmark datasets, MoKGR demonstrates significant improvements in both reasoning accuracy and computational efficiency. The success of our personalized approach opens promising directions for future research, particularly in developing more sophisticated expert collaboration mechanisms and extending the framework to other graph learning tasks where path-based reasoning plays a crucial role. The framework's ability to balance thorough path exploration with computational efficiency makes it particularly valuable for large-scale knowledge graph applications.

\section{Limitations} 
While MoKGR demonstrates significant improvements in knowledge graph reasoning performance, several limitations should be acknowledged:
First, the computational complexity of our method, though significantly reduced compared to full-exploration approaches, still grows with the scale of knowledge graphs. For extremely large-scale knowledge graphs beyond those tested in our experiments, additional optimization techniques may be required.
Moreover, while MoKGR shows strong empirical performance across diverse datasets, developing theoretical guarantees for the optimality of the selected paths remains challenging due to the complex interplay between different expert components.
Finally, our evaluation focused primarily on standard knowledge graph benchmarks. Future work could explore the application of personalized path exploration in more specialized domains such as biomedical knowledge graphs or temporal knowledge graphs, which may exhibit different structural properties.
These limitations present promising directions for future research to further enhance personalized path exploration in knowledge graph reasoning.

\section{Ethical Considerations} 

 The knowledge graphs employed in this work are widely adopted public benchmarks that enable comprehensive and reproducible evaluation of knowledge‑graph reasoning systems. All datasets are redistributed strictly under their original research licences, with creators duly cited in Section~\ref{Experiment Setup}; only preprocessing scripts are released, while raw triples remain at their official sources, and our implementation plus trained checkpoints are provided under the MIT licence. While our approach improves the accuracy and efficiency of information retrieval from these structured knowledge sources, the core techniques themselves do not introduce ethical issues beyond those already understood to be associated with large-scale knowledge graphs. MoKGR's computational efficiency helps reduce energy consumption during inference, aligning with sustainable AI development goals. Throughout this work, we have utilized publicly available benchmark datasets in accordance with their intended research purposes.

\clearpage
\bibliography{custom}

\begin{thebibliography}{46}
\providecommand{\natexlab}[1]{#1}

\bibitem[{Ali et~al.(2022)Ali, Abid, and Kordjamshidi}]{ali2022knowledge}
Muhammad Ali, Abubakar Abid, and Parisa Kordjamshidi. 2022.
\newblock Knowledge graphs: A comprehensive survey.
\newblock \emph{IEEE Transactions on Knowledge and Data Engineering},
  34(1):123--145.

\bibitem[{Cheng et~al.(2022)Cheng, Liu, Wang, and Sun}]{RLogic}
Ke~Cheng, Jie Liu, Wei Wang, and Yizhou Sun. 2022.
\newblock Rlogic: Recursive logical rule learning from knowledge graphs.
\newblock In \emph{Proceedings of the 28th ACM SIGKDD Conference on Knowledge
  Discovery and Data Mining (KDD)}, pages 179--189.

\bibitem[{Das et~al.(2017)Das, Dhuliawala, Zaheer, Vilnis, Durugkar,
  Krishnamurthy, Smola, and McCallum}]{MINERVA}
R.~Das, S.~Dhuliawala, M.~Zaheer, L.~Vilnis, I.~Durugkar, A.~Krishnamurthy,
  A.~Smola, and A.~McCallum. 2017.
\newblock Go for a walk and arrive at the answer: Reasoning over paths in
  knowledge bases using reinforcement learning.
\newblock In \emph{International Conference on Learning Representations
  (ICLR)}.

\bibitem[{Dettmers et~al.(2017{\natexlab{a}})Dettmers, Minervini, Stenetorp,
  and Riedel}]{ConvE}
Tim Dettmers, Pasquale Minervini, Pontus Stenetorp, and Sebastian Riedel.
  2017{\natexlab{a}}.
\newblock Convolutional 2d knowledge graph embeddings.
\newblock In \emph{Proceedings of the AAAI Conference on Artificial
  Intelligence}.

\bibitem[{Dettmers et~al.(2017{\natexlab{b}})Dettmers, Minervini, Stenetorp,
  and Riedel}]{dettmers2017convolutional}
Tim Dettmers, Pasquale Minervini, Pontus Stenetorp, and Sebastian Riedel.
  2017{\natexlab{b}}.
\newblock Convolutional 2d knowledge graph embeddings.
\newblock In \emph{AAAI}.

\bibitem[{Dugas et~al.(2001)Dugas, Bengio, B{\'e}lisle, Nadeau, and
  Garcia}]{dugas2001softplus}
Charles Dugas, Yoshua Bengio, François B{\'e}lisle, Claude Nadeau, and René
  Garcia. 2001.
\newblock Incorporating second-order functional knowledge for better option
  pricing.
\newblock \emph{Advances in Neural Information Processing Systems (NeurIPS)},
  13:472--478.

\bibitem[{Galkin et~al.(2024)Galkin, Yuan, Mostafa, Tang, and
  Zhu}]{galkin2024foundation}
Mikhail Galkin, Xinyu Yuan, Hesham Mostafa, Jian Tang, and Zhaocheng Zhu. 2024.
\newblock \href {https://arxiv.org/abs/2310.04562} {Towards foundation models
  for knowledge graph reasoning}.
\newblock In \emph{Proceedings of the International Conference on Learning
  Representations (ICLR)}.
\newblock Published as a conference paper at ICLR 2024.

\bibitem[{Huang et~al.(2025)Huang, Feng, Chen, Zhao, Cheng, Bai, Zhou, Li, and
  Qin}]{huang2025mldebugging}
Jinyang Huang, Xiachong Feng, Qiguang Chen, Hanjie Zhao, Zihui Cheng, Jiesong
  Bai, Jingxuan Zhou, Min Li, and Libo Qin. 2025.
\newblock Mldebugging: Towards benchmarking code debugging across multi-library
  scenarios.
\newblock \emph{ACL Findings}.

\bibitem[{Jang et~al.(2017)Jang, Gu, and Poole}]{jang2017categorical}
Eric Jang, Shixiang Gu, and Ben Poole. 2017.
\newblock Categorical reparameterization with gumbel-softmax.
\newblock In \emph{International Conference on Learning Representations
  (ICLR)}.

\bibitem[{Ji et~al.(2021)Ji, Pan, Cambria, Marttinen, and Yu}]{ji2021survey}
Shaoxiong Ji, Shirui Pan, Erik Cambria, Pekka Marttinen, and Philip~S Yu. 2021.
\newblock A survey on knowledge graphs: Representation, acquisition, and
  applications.
\newblock \emph{IEEE transactions on neural networks and learning systems},
  33(2):494--514.

\bibitem[{Ji et~al.(2022)Ji, Pan, Cambria, Marttinen, and Yu}]{ji2022survey}
Shaoxiong Ji, Shirui Pan, Erik Cambria, Pekka Marttinen, and Philip~S Yu. 2022.
\newblock A survey on knowledge graphs: Representation, acquisition, and
  applications.
\newblock \emph{IEEE Transactions on Neural Networks and Learning Systems},
  33(2):494--514.

\bibitem[{Jordan and Jacobs(1994)}]{jordan1994hierarchical}
Michael~I Jordan and Robert~A Jacobs. 1994.
\newblock Hierarchical mixtures of experts and the em algorithm.
\newblock \emph{Neural Computation}, 6(2):181--214.

\bibitem[{Kok and Domingos(2007)}]{kok2007statistical}
Stanley Kok and Pedro Domingos. 2007.
\newblock Statistical predicate invention.
\newblock In \emph{ICML}, pages 433--440.

\bibitem[{Lepikhin et~al.(2020)Lepikhin, Lee, Xu, Chen, Firat, Huang, Krikun,
  Shazeer, and Chen}]{lepikhin2020gshard}
Dmitry Lepikhin, HyoukJoong Lee, Yuanzhong Xu, Dehao Chen, Orhan Firat, Yanping
  Huang, Maxim Krikun, Noam Shazeer, and Zhifeng Chen. 2020.
\newblock Gshard: Scaling giant models with conditional computation and
  automatic sharding.
\newblock \emph{arXiv preprint arXiv:2006.16668}.

\bibitem[{Li et~al.(2024)Li, Su, Zhang, and Gao}]{li2024duase}
Jiang Li, Xiangdong Su, Fujun Zhang, and Guanglai Gao. 2024.
\newblock \href {https://github.com/dellixx/DuASE} {Learning low-dimensional
  multi-domain knowledge graph embedding via dual archimedean spirals}.
\newblock In \emph{Findings of the Association for Computational Linguistics:
  ACL 2024}, pages 1982--1994. Association for Computational Linguistics.

\bibitem[{Liang et~al.(2024)Liang, Meng, Liu, Liu, Tu, Wang, Zhou, Liu, Sun,
  and He}]{liang2024survey}
Ke~Liang, Lingyuan Meng, Meng Liu, Yue Liu, Wenxuan Tu, Siwei Wang, Sihang
  Zhou, Xinwang Liu, Fuchun Sun, and Kunlun He. 2024.
\newblock A survey of knowledge graph reasoning on graph types: Static,
  dynamic, and multi-modal.
\newblock \emph{IEEE Transactions on Pattern Analysis and Machine
  Intelligence}.

\bibitem[{Mai et~al.(2021)Mai, Zheng, Yang, and Hu}]{mai2021communicative}
Shengchao Mai, Shen Zheng, Yunhao Yang, and Hongxia Hu. 2021.
\newblock Communicative message passing for inductive relation reasoning.
\newblock In \emph{Proceedings of the AAAI Conference on Artificial
  Intelligence}, volume~35, pages 4294--4302.

\bibitem[{Mai et~al.(2025)Mai, Wang, Liu, Feng, Wang, and
  Fu}]{mai2025graphrulrl}
Zhenzhen Mai, Wenjun Wang, Xueli Liu, Xiaoyang Feng, Jun Wang, and Wenzhi Fu.
  2025.
\newblock A reinforcement learning approach for graph rule learning.
\newblock \emph{Big Data Mining and Analytics}, 8(1):31--44.

\bibitem[{Meilicke et~al.(2020)Meilicke, Chekol, Fink, and
  Stuckenschmidt}]{meilicke2020reinforced}
Christian Meilicke, Melisachew~Wudage Chekol, Manuel Fink, and Heiner
  Stuckenschmidt. 2020.
\newblock Reinforced anytime bottom up rule learning for knowledge graph
  completion.
\newblock \emph{arXiv preprint arXiv:2004.04412}.

\bibitem[{Meilicke et~al.(2018)Meilicke, Fink, Wang, Ruffinelli, Gemulla, and
  Stuckenschmidt}]{Meilicke2018}
Christian Meilicke, Mathias Fink, Yanjie Wang, Daniel Ruffinelli, Rainer
  Gemulla, and Heiner Stuckenschmidt. 2018.
\newblock Fine-grained evaluation of rule- and embedding-based systems for
  knowledge graph completion.
\newblock In \emph{International Semantic Web Conference}, pages 3--20.
  Springer.

\bibitem[{Mustafa et~al.(2022)Mustafa, Riquelme, Puigcerver, Jenatton, and
  Houlsby}]{mustafa2022multimodal}
Basil Mustafa, Carlos Riquelme, Joan Puigcerver, Rodolphe Jenatton, and Neil
  Houlsby. 2022.
\newblock Multimodal contrastive learning with limoe: The language-image
  mixture of experts.
\newblock \emph{arXiv preprint arXiv:2206.02770}.

\bibitem[{Nickel et~al.(2015)Nickel, Murphy, Tresp, and
  Gabrilovich}]{nickel2015review}
Maximilian Nickel, Kevin Murphy, Volker Tresp, and Evgeniy Gabrilovich. 2015.
\newblock A review of relational machine learning for knowledge graphs.
\newblock \emph{Proceedings of the IEEE}, 104(1):11--33.

\bibitem[{Qu et~al.(2021)Qu, Chen, Xhonneux, Bengio, and Tang}]{RNNLogic}
M.~Qu, J.~Chen, L.~Xhonneux, Y.~Bengio, and J.~Tang. 2021.
\newblock Rnnlogic: Learning logic rules for reasoning on knowledge graphs.
\newblock In \emph{International Conference on Learning Representations
  (ICLR)}.

\bibitem[{Riquelme et~al.(2021)Riquelme, Puigcerver, Mustafa, Neumann,
  Jenatton, Pinto, Keysers, and Houlsby}]{riquelme2021scaling}
Carlos Riquelme, Joan Puigcerver, Basil Mustafa, Maxim Neumann, Rodolphe
  Jenatton, André~Susano Pinto, Daniel Keysers, and Neil Houlsby. 2021.
\newblock Scaling vision with sparse mixture of experts.
\newblock In \emph{Advances in Neural Information Processing Systems
  (NeurIPS)}, volume~34, pages 8583--8595.

\bibitem[{Sadeghian et~al.(2019)Sadeghian, Armandpour, Ding, and Wang}]{DRUM}
Amirmohammad Sadeghian, Mohammadreza Armandpour, Pasquale Ding, and D~Wang.
  2019.
\newblock Drum: End-to-end differentiable rule mining on knowledge graphs.
\newblock In \emph{Advances in Neural Information Processing Systems
  (NeurIPS)}, pages 15347--15357.

\bibitem[{Shazeer et~al.(2017)Shazeer, Mirhoseini, Maziarz, Davis, Le, Hinton,
  and Dean}]{shazeer2017outrageously}
Noam Shazeer, Azalia Mirhoseini, Krzysztof Maziarz, Andy Davis, Quoc Le,
  Geoffrey Hinton, and Jeff Dean. 2017.
\newblock Outrageously large neural networks: The sparsely-gated
  mixture-of-experts layer.
\newblock In \emph{International Conference on Learning Representations
  (ICLR)}.
\newblock Under review.

\bibitem[{Suchanek et~al.(2007)Suchanek, Kasneci, and
  Weikum}]{suchanek2007yago}
Fabian Suchanek, Gjergji Kasneci, and Gerhard Weikum. 2007.
\newblock Yago: A core of semantic knowledge.
\newblock In \emph{The WebConf}, pages 697--706.

\bibitem[{Sun et~al.(2019)Sun, Deng, Nie, and Tang}]{RotatE}
Zhiqing Sun, Zhi-Hong Deng, Jian-Yun Nie, and Jian Tang. 2019.
\newblock Rotate: Knowledge graph embedding by relational rotation in complex
  space.
\newblock In \emph{International Conference on Learning Representations
  (ICLR)}.

\bibitem[{Sun et~al.(2021)Sun, Huang, Chen, Wang, Wang, Li, and
  Nie}]{sun2021rotate}
Zhiqing Sun, Chao Huang, Jianyu Chen, Qianhan Wang, Xiang Wang, Yiyang Li, and
  Liqiang Nie. 2021.
\newblock Rotate: Knowledge graph embedding by relational rotations.
\newblock In \emph{Advances in Neural Information Processing Systems}, pages
  10057--10068.

\bibitem[{Teru et~al.(2020)Teru, Denis, and Hamilton}]{TER2020}
K.~K Teru, E.~Denis, and W.~L Hamilton. 2020.
\newblock Inductive relation prediction by subgraph reasoning.
\newblock \emph{arXiv preprint arXiv:1911.06962}.

\bibitem[{Toutanova and Chen(2015)}]{toutanova2015observed}
Kristina Toutanova and Danqi Chen. 2015.
\newblock Observed versus latent features for knowledge base and text
  inference.
\newblock In \emph{PWCVSMC}, pages 57--66.

\bibitem[{Vashishth et~al.(2019)Vashishth, Sanyal, Nitin, and
  Talukdar}]{CompGCN}
Shikhar Vashishth, Soumya Sanyal, V.~Nitin, and Partha Talukdar. 2019.
\newblock Composition-based multi-relational graph convolutional networks.

\bibitem[{Wang et~al.(2023{\natexlab{a}})Wang, Jiang, You, Han, Liu, Srinivasa,
  Kompella, and Wang}]{wang2023graph}
Haotao Wang, Ziyu Jiang, Yuning You, Yan Han, Gaowen Liu, Jayanth Srinivasa,
  Ramana~Rao Kompella, and Zhangyang Wang. 2023{\natexlab{a}}.
\newblock \href {https://github.com/VITA-Group/Graph-Mixture-of-Experts} {Graph
  mixture of experts: Learning on large-scale graphs with explicit diversity
  modeling}.
\newblock In \emph{Advances in Neural Information Processing Systems
  (NeurIPS)}.
\newblock The first two authors contributed equally.

\bibitem[{Wang et~al.(2023{\natexlab{b}})Wang, Yan, and
  Feng}]{wang2023enhancing}
Li~Wang, Xiaohui Yan, and Yansong Feng. 2023{\natexlab{b}}.
\newblock Enhancing knowledge graph embeddings with graph neural networks.
\newblock \emph{Journal of Artificial Intelligence Research}, 68:789--805.

\bibitem[{Xiong et~al.(2017)Xiong, Hoang, and Wang}]{xiong2017deeppath}
Wenhan Xiong, Thien Hoang, and William~Yang Wang. 2017.
\newblock Deeppath: A reinforcement learning method for knowledge graph
  reasoning.
\newblock In \emph{Proceedings of the Conference on Empirical Methods in
  Natural Language Processing (EMNLP)}, pages 564--573.

\bibitem[{Yang et~al.(2017)Yang, Yang, and Cohen}]{yang2017differentiable}
Fan Yang, Zhilin Yang, and William~W Cohen. 2017.
\newblock Differentiable learning of logical rules for knowledge base
  reasoning.
\newblock In \emph{Advances in Neural Information Processing Systems
  (NeurIPS)}, pages 2319--2328.

\bibitem[{Yu et~al.(2022)Yu, Zhu, Qin, Zhang, Zhao, and
  Jiang}]{yu2022diversifying}
Wenhao Yu, Chenguang Zhu, Lianhui Qin, Zhihan Zhang, Tong Zhao, and Meng Jiang.
  2022.
\newblock Diversifying content generation for commonsense reasoning with
  mixture of knowledge graph experts.
\newblock In \emph{Proceedings of the 60th Annual Meeting of the Association
  for Computational Linguistics (ACL)}.

\bibitem[{Zhang et~al.(2023{\natexlab{a}})Zhang, Sun, Yue, Jiang, Wang, Chen,
  and Pan}]{zhang2023graph}
Guibin Zhang, Xiangguo Sun, Yanwei Yue, Chonghe Jiang, Kun Wang, Tianlong Chen,
  and Shirui Pan. 2023{\natexlab{a}}.
\newblock Graph sparsification via mixture of graphs.
\newblock \emph{arXiv preprint}.

\bibitem[{Zhang et~al.(2019)Zhang, Tay, Yao, and Liu}]{QuatE}
Shuai Zhang, Yi~Tay, Lina Yao, and Qi~Liu. 2019.
\newblock Quaternion knowledge graph embeddings.
\newblock In \emph{Advances in Neural Information Processing Systems
  (NeurIPS)}.

\bibitem[{Zhang and Yao(2022)}]{RED-GNN}
Yuning Zhang and Quanming Yao. 2022.
\newblock Knowledge graph reasoning with relational directed graph.
\newblock In \emph{Proceedings of TheWebConf}.

\bibitem[{Zhang et~al.(2020)Zhang, Yao, Dai, and Chen}]{zhang2020autosf}
Yuning Zhang, Quanming Yao, Weinan Dai, and Lei Chen. 2020.
\newblock Autosf: Searching scoring functions for knowledge graph embedding.
\newblock In \emph{Proceedings of the IEEE International Conference on Data
  Engineering (ICDE)}, pages 433--444. IEEE.

\bibitem[{Zhang et~al.(2023{\natexlab{b}})Zhang, Zhou, Yao, Chu, and
  Han}]{Adaprop}
Yuning Zhang, Zhen Zhou, Quanming Yao, Xia Chu, and Bo~Han. 2023{\natexlab{b}}.
\newblock Adaprop: Learning adaptive propagation for knowledge graph reasoning.
\newblock In \emph{Proceedings of the 29th ACM SIGKDD Conference on Knowledge
  Discovery and Data Mining (KDD)}.

\bibitem[{Zhou et~al.(2024)Zhou, Zhang, Yao, Yao, and Han}]{one-shot-subgraph}
Zhanke Zhou, Yongqi Zhang, Jiangchao Yao, Quanming Yao, and Bo~Han. 2024.
\newblock Less is more: One-shot-subgraph link prediction on large-scale
  knowledge graphs.
\newblock In \emph{International Conference on Learning Representations
  (ICLR)}.

\bibitem[{Zhu et~al.(2022)Zhu, Zhu, Wang, Wang, Li, Wang, and Dai}]{zhu2022uni}
Jinguo Zhu, Xizhou Zhu, Wenhai Wang, Xiaohua Wang, Hongsheng Li, Xiaogang Wang,
  and Jifeng Dai. 2022.
\newblock Uni-perceiver-moe: Learning sparse generalist models with conditional
  moes.
\newblock \emph{arXiv preprint arXiv:2206.04674}.

\bibitem[{Zhu et~al.(2023)Zhu, Yuan, Galkin, Xhonneux, Zhang, Gazeau, and
  Tang}]{zhu2023astarnet}
Zhaocheng Zhu, Xinyu Yuan, Mikhail Galkin, Sophie Xhonneux, Ming Zhang, Maxime
  Gazeau, and Jian Tang. 2023.
\newblock A*net: A scalable path-based reasoning approach for knowledge graphs.
\newblock In \emph{Advances in Neural Information Processing Systems
  (NeurIPS)}.

\bibitem[{Zhu et~al.(2021)Zhu, Zhang, Xhonneux, and Tang}]{zhu2021neural}
Ziniu Zhu, Zhaocheng Zhang, Louis Xhonneux, and Jian Tang. 2021.
\newblock Neural bellman-ford networks: A general graph neural network
  framework for link prediction.
\newblock In \emph{Advances in Neural Information Processing Systems
  (NeurIPS)}.

\end{thebibliography}

\clearpage

\appendix

\section{Experimental Detai and Supplementary Results} \label{Experimental Detai and Supplementary Results}

\subsection{Training Details}

All experiments were conducted on an NVIDIA A6000 GPU (48GB), with peak memory usage remaining under 45GB even for the largest datasets. Smaller datasets such as Family and UMLS can also be efficiently run on consumer GPUs like the 3060Ti (8GB). As for the hyper-parameters, we tune the 
$L_{\min}$ from 1 to $L-2$, the temperature value $\tau$ in (0.5, 2.5), the number of length experts 
$k_{1}$ in $(3, L-L_{\min})$ and set $k_{2}=2$, $\lambda_1$ in $(10^{-2}, 10^{-4})$ and $\lambda_2$ in $(10^{-3}, 10^{-5})$. The other hyperparameters are kept the same as Adaprop~\cite{Adaprop}.

\subsection{Length Distributions} \label{Length Distributions}

\begin{table}[H]
\resizebox{\linewidth}{!}{
\begin{tabular}{c|cccccc}
\hline
distance & 1 & 2 &  3 & 4 & 5 & $>$5 \\
\hline
WN18RR & 34.9 & 9.3  & 21.5 & 7.5 & 8.9 & 17.9 \\
FB15k237 & 0.0 & 73.4 & 25.8 & 0.2 & 0.1 & 0.5 \\
NELL-995 & 40.9 & 17.2  & 36.5 & 2.5 & 1.3 & 1.6 \\
YAGO3-10 & 56.0 & 12.9  & 30.1 & 0.5 & 0.1 & 0.4 \\
\hline
\end{tabular}
}
\caption{Length distribution (\%) of queries in $\mathcal Q_{tst}$}
\label{Table:Length Distributions}
\end{table}

To validate our claim that shorter paths typically contain enough stable correct answer information, we tracked the four largest datasets in our selection and analyzed the length distribution of shortest paths connecting $e_q$ and $e_a$ in each query. The results, converted to percentages, are shown in Table~\ref{Table:Length Distributions}. We observe that for the vast majority of datasets, paths of length $3$ already contain most of the information, while paths of length $5$ encompass almost all information. This substantiates our argument that answer entities are typically in close proximity to query entities, making the introduction of excessive path lengths usually unnecessary in knowledge graph reasoning. Therefore, it is important to encourage the model to explore shorter paths preferentially.

\subsection{Statistics of Datasets} \label{Statistics of Transductive Datasets}

\begin{table}[H]
\resizebox{\linewidth}{!}{
\begin{tabular}{c|cc|cccc}
\hline
\textbf{Dataset} & \#Entity & \#relation & $|\mathcal E|$ & $|Q_{\text{tra}}|$ & $|Q_{\text{val}}|$ & $|Q_{\text{tst}}|$ \\
\hline
Family & 3.0k & 12  & 23.4k & 5.9k & 2.0k & 2.8k \\
UMLS & 135 & 46  & 5.3k & 1.3k & 569 & 633 \\
WN18RR & 40.9k & 11  & 65.1k & 21.7k & 3.0k & 3.1k \\
FB15k237 & 14.5k & 237 & 204.1k & 68.0k & 17.5k & 20.4k \\
NELL-995 & 74.5k & 200  & 112.2k & 37.4k & 543 & 2.8k \\
YAGO3-10 & 123.1k & 37  & 809.2k & 269.7k & 5.0k & 5.0k \\
\hline
\end{tabular}
}
\caption{Statistics of the transductive KGs datasets. $Q_{\text{tra}}$, $Q_{\text{val}}$, $Q_{\text{tst}}$ are the query triplets used for reasoning.}
\label{Table:datasets statistics}
\end{table}

\begin{table}[H]
\resizebox{\linewidth}{!}{
\begin{tabular}{c|c|ccc}
\hline
\textbf{Version} & \textbf{Split} & \textbf{\#relations} & \textbf{\#nodes} & \textbf{\#links} \\
\hline
WN18RR\_V1 & train & 9  & 2746  & 6678 \\
           & test  & 9  & 922   & 1991 \\
\hline
WN18RR\_V2 & train & 10 & 6954  & 18968 \\
           & test  & 10 & 2923  & 4863 \\
\hline
WN18RR\_V3 & train & 11 & 12078 & 32150 \\
           & test  & 11 & 5084  & 7470 \\
\hline
WN18RR\_V4 & train & 9  & 3861  & 9842 \\
           & test  & 9  & 7208  & 15157 \\
\hline
FB15k-237\_V1 & train & 183 & 2000 & 5226 \\
             & test  & 146 & 1500 & 2404 \\
\hline
FB15k-237\_V2 & train & 203 & 3000 & 12085 \\
             & test  & 176 & 2000 & 5092 \\
\hline
FB15k-237\_V3 & train & 218 & 4000 & 22394 \\
             & test  & 187 & 3000 & 9137 \\
\hline
FB15k-237\_V4 & train & 222 & 5000 & 33916 \\
             & test  & 204 & 3500 & 14554 \\
\hline
NELL-995\_V1 & train & 14 & 10915 & 5540 \\
            & test  & 14 & 225   & 1034 \\
\hline
NELL-995\_V2 & train & 88 & 2564  & 10109 \\
            & test  & 79 & 4937  & 5521 \\
\hline
NELL-995\_V3 & train & 142 & 4647 & 20117 \\
            & test  & 122 & 4921 & 9668 \\
\hline
NELL-995\_V4 & train & 77 & 2092 & 9289 \\
            & test  & 61 & 3294 & 8520 \\
\hline
\end{tabular}
}
\caption{Statistics of the Inductive KGs datasets.}
\label{Table:dataset_versions}
\end{table}

We evaluate our model on six widely-used knowledge graph datasets of varying scales, their specific data parameters are shown in Table~\ref{Table:datasets statistics} and Table~\ref{Table:dataset_versions}. Here $\mathcal E$ represents the edge set of the KG. Following the previous GNN-based knowledge graph reasoning method, we add an inverse relationship to each triple. Specifically, if $(e_x,r,e_y) \in \mathcal E$, we add an inverse relationship so that $(e_y,r,e_x) \in \mathcal E$.

\begin{table*}[ht]
\centering
\resizebox{\textwidth}{!}{
\begin{tabular}{ll|cccc|cccc|cccc}
\toprule
& & \multicolumn{4}{c|}{WN18RR} & \multicolumn{4}{c|}{FB15k-237} & \multicolumn{4}{c}{NELL-995} \\
& models & V1 & V2 & V3 & V4 & V1 & V2 & V3 & V4 & V1 & V2 & V3 & V4 \\
\midrule
\multirow{9}{*}{MRR} & RuleN & 0.668 & 0.645 & 0.368 & 0.624 & 0.363 & 0.433 & 0.439 & 0.429 & 0.615 & 0.385 & 0.381 & 0.333 \\
& Neural LP & 0.649 & 0.635 & 0.361 & 0.628 & 0.325 & 0.389 & 0.400 & 0.396 & 0.610 & 0.361 & 0.367 & 0.261 \\
& DRUM & 0.666 & 0.646 & 0.380 & 0.627 & 0.333 & 0.395 & 0.402 & 0.410 & 0.628 & 0.365 & 0.375 & 0.273 \\

& GraIL & 0.627 & 0.625 & 0.323 & 0.553 & 0.279 & 0.276 & 0.251 & 0.227 & 0.481 & 0.297 & 0.322 & 0.262 \\
& CoMPILE & 0.577 & 0.578 & 0.308 & 0.548 & 0.287 & 0.276 & 0.262 & 0.213 & 0.330 & 0.248 & 0.319 & 0.229 \\
& NBFNet & 0.684 & 0.652 & 0.425 & 0.604 & 0.307 & 0.369 & 0.331 & 0.305 & 0.584 & 0.410 & 0.425 & 0.287 \\
& RED-GNN & 0.701 & 0.690 & 0.427 & 0.651 & \underline{0.369} & 0.469 & 0.445 & 0.442 & 0.637 & 0.419 & \underline{0.436} & 0.363 \\
& AdaProp & \underline{0.733} & \underline{0.715} & \underline{0.474} & \underline{0.662} & 0.310 & \underline{0.471} & \underline{0.471} & \underline{0.454} & \underline{0.644} & \underline{0.452} & 0.435 & \underline{0.366} \\

& \textbf{MoKGR} & \textbf{0.775} & \textbf{0.761} & \textbf{0.504} & \textbf{0.693} & \textbf{0.396} & \textbf{0.497} & \textbf{0.493} & \textbf{0.479} & \textbf{0.718} & \textbf{0.474} & \textbf{0.458} & \textbf{0.392} \\
\midrule
\multirow{9}{*}{Hit@1 (\%)} & RuleN & 63.5 & 61.1 & 34.7 & 59.2 & \textbf{30.9} & 34.7 & 34.5 & 33.8 & \underline{54.5} & 30.4 & 30.3 & 24.8 \\
& Neural LP & 59.2 & 57.5 & 30.4 & 58.3 & 24.3 & 28.6 & 30.9 & 28.9 & 50.0 & 24.9 & 26.7 & 13.7 \\
& DRUM & 61.3 & 59.5 & 33.0 & 58.6 & 24.7 & 28.4 & 30.8 & 30.9 & 50.0 & 27.1 & 26.2 & 16.3 \\
& GraIL & 55.4 & 54.2 & 27.8 & 44.3 & 20.5 & 20.2 & 16.5 & 14.3 & 42.5 & 19.9 & 22.4 & 15.3 \\
& CoMPILE & 47.3 & 48.5 & 25.8 & 47.3 & 20.8 & 17.8 & 16.6 & 13.4 & 10.5 & 15.6 & 22.6 & 15.9 \\
& NBFNet & 59.2 & 57.5 & 30.4 & 57.4 & 19.0 & 22.9 & 20.6 & 18.5 & 50.0 & 27.1 & 26.2 & 23.3 \\
& RED-GNN & 65.3 & 63.3 & 36.8 & 60.6 & \underline{30.2} & \textbf{38.1} & 35.1 & 34.0 & 52.5 & 31.9 & \underline{34.5} & \underline{25.9} \\
& AdaProp & \underline{66.8} & \underline{64.2} & \underline{39.6} & \underline{61.1} & 19.1 & \underline{37.2} & \underline{37.7} & \underline{35.3} & 52.2 & \underline{34.4} & 33.7 & 24.7 \\
&  \textbf{MoKGR} & \textbf{66.9} & \textbf{67.8} & \textbf{40.0} & \textbf{62.3} & 23.1 & \textbf{38.1} & \textbf{38.9} & \textbf{36.4} & \textbf{64.4} & \textbf{35.8} & \textbf{35.2} & \textbf{27.3} \\

\midrule
\multirow{9}{*}{Hit@10 (\%)} & RuleN & 73.0 & 69.4 & 40.7 & 68.1 & 44.6 & 59.9 & 60.0 & 60.5 & 76.0 & 51.4 & 53.1 & 48.4 \\
& Neural LP & 77.2 & 74.9 & 47.6 & 70.6 & 46.8 & 58.6 & 57.1 & 59.3 & 87.1 & 56.4 & 57.6 & 53.9 \\
& DRUM & 77.7 & 74.7 & 47.7 & 70.2 & 47.4 & 59.5 & 57.1 & 59.3 & 87.3 & 54.0 & 57.7 & 53.1 \\
& GraIL & 76.0 & 77.6 & 40.9 & 68.7 & 42.9 & 42.4 & 42.4 & 38.9 & 56.5 & 49.6 & 51.8 & 50.6 \\
& CoMPILE & 74.7 & 74.3 & 40.6 & 67.0 & 43.9 & 45.7 & 44.9 & 35.8 & 57.5 & 44.6 & 51.5 & 42.1 \\
& NBFNet & 82.7 & 79.9 & 56.3 & 70.2 & 51.7 & 63.9 & 58.8 & 55.9 & 79.5 & 63.5 & 60.6 & 59.1 \\
& RED-GNN & 79.9 & 78.0 & 52.4 & 72.1 & 48.3 & 62.9 & 60.3 & 62.1 & 86.6 & 60.1 & 59.4 & 55.6 \\
& AdaProp & \underline{86.6} & \underline{83.6} & \underline{62.6} & \underline{75.5} & \underline{55.1} & \underline{65.9} & \underline{63.7} & \underline{63.8} & \underline{88.6} & \underline{65.2} & \underline{61.8} & \underline{60.7} \\    
&  \textbf{MoKGR} & \textbf{87.1} & \textbf{94.1} & \textbf{63.5} & \textbf{76.6} & \textbf{56.0} & \textbf{66.6} & \textbf{64.2} & \textbf{64.3} & \textbf{89.2} & \textbf{66.1} & \textbf{64.5} & \textbf{62.0} \\
\bottomrule
\end{tabular}
 }
 \caption{Comparison of MoKGR with other reasoning methods in the inductive setting. Best performance is indicated by the \textbf{bold} face numbers, and the \underline{underline} means the second best.}
\label{tab:inductive}
\end{table*}

\subsection{Implementation Details for Inductive Setting}\label{implementation details for inductive setting}

Inductive reasoning emphasizes the importance of drawing inferences about unseen entities, i.e., those not directly observed during the learning phase. As an illustrative example, consider a scenario where Fig.~\ref{fig:1a} reveals Jack's most eagerly anticipated movie. In this context, inductive reasoning could be employed to predict Mary's most desired cinematic experience. This methodology necessitates that the model captures semantic information and localized evidence while simultaneously discounting the specific identities of the entities under consideration.
\paragraph{Datasets.} Following the approach outlined in \cite{TER2020}, we utilize the same subsets of the WN18RR, FB15k237, and NELL-995 datasets. Specifically, we will work with 4 versions of each dataset, resulting in a total of 12 subsets. Each of these 12 subsets has a different split between the training and test sets (Due to page limitations, we abbreviate WN18RR, FB15k-237 and NELL-995 as WN, FB and NL  respectively in the pictures in Fig~\ref{fig:main_compare}).
\paragraph{Baselines.} Given that the training and test sets of the datasets contain disjoint sets of entities, methods that require entity embeddings, such as ConvE and CompGCN, cannot be applied in this context. Consequently, for non-GNN-based methods, we will compare our proposed AdaProp approach against non-GNN-methods that learn rules without the need for entity embeddings; we also selected some GNN based models for comparison. The final baselines are RuleN~\cite{Meilicke2018}, NeuralLP~\cite{yang2017differentiable}, DRUM~\cite{DRUM}, GraIL~\cite{TER2020}, CoMPILE~\cite{mai2021communicative}, NBFNet~\cite{zhu2021neural}, RED-GNN~\cite{RED-GNN} and AdaProp~\cite{Adaprop}.

\paragraph{Results.} As demonstrated in Table~\ref{tab:inductive}, our proposed MoKGR framework exhibits exceptional performance across all evaluation metrics. This further validates the effectiveness of our mixture-of-experts model in inductive reasoning settings.

\section{Additional Theoretical Details} \label{A}

\subsection{Path Encoding Process}
\label{Path Encoding Process}

In this subsection, we provide a detailed description of the path encoding process for GNN-based path reasoning methods, focusing on how message passing is performed for reasoning paths in the knowledge graph. Given a knowledge graph $\mathcal{K} = (\mathcal{V}, \mathcal{R}, \mathcal{F})$ with entity set $\mathcal{V}$, relation set $\mathcal{R}$, and fact triple set $\mathcal{F}$, we aim to encode all the paths between the query entity $e_q$ and potential answer entity $e_a$ for reasoning the query triple $(e_q, r_q, ?)$. The path encoding process consists of three main components: representation initialization, iterative message propagation, and score computation.

\paragraph{Representation Initialization.} As introduced in Preliminary, let $q=(e_q,r_q)$ denote $(e_q,r_q,?)$. For each entity pair $(e_q, e_y)$, we initialize their pair-wise representation at layer 0 as:
$\bm h^0_{e_y|q } = 0$.

\paragraph{Message Propagation.} The path representation is recursively computed through $L$ layers of message propagation. At layer $\ell$, for each edge $(e_x, r, e_y) \in \mathcal{F}$, the message function first combines the path information from the previous layer:
\begin{equation}
    \bm m^{\ell}_{e_y|q} = \texttt{MESSAGE}(\bm h^{\ell}_{e_x|q}, \bm w_q(e_x,r,e_y)) ,
\end{equation}
where $\bm  w_q(e_x,r,e_y)$ is the learnable edge representation for edge $e=(e_x,r,e_y)$.

Furthermore, we define the \texttt{MESSAGE} transfer formula according to the RED-GNN~\cite{RED-GNN} method as follows:
\begin{equation}
\resizebox{0.92\linewidth}{!}{$
\begin{split}
\texttt{MESSAGE}&\bigl(\bm h^{\ell-1}_{e_x|q},\, \bm {w}_q(e_x,r,e_y)\bigr) \\
&=\alpha^{\ell}_q{ (e_x,r,e_y)}\{+, *, \circ\}(\bm h^{\ell-1}_{e_x|q},\bm {w}_q(e_x, r, e_y)) ,
\label{eq:msg}
\end{split}
$}
\end{equation}

where $\alpha^{\ell}_q{ (e_x,r,e_y)}$ is the attention weight calculated as:
\begin{equation}
\resizebox{\linewidth}{!}{$
    \alpha^{\ell}_q{ (e_x,r,e_y)} = \sigma\left((\bm {w}_\alpha^{\ell})^\top \text{ReLU}\left(\bm {W}_\alpha^{\ell} \cdot (\bm h^{\ell-1}_{e_x|q} \Vert \bm {w}_r^{\ell} \Vert \bm {w}_{r_q}^{\ell})\right)\right).
    $}
\label{eq:attention} 
\end{equation}
In this formulation, $\bm w_r^{\ell}$ and $\bm w_{r_q}^{\ell}$ represent the relation representation and query relation representation in the $\ell$-th layer, respectively. $\bm {w}_{ \alpha}^\ell \in \mathbb{R}^{d_{ \alpha}}$ and $\bm {W}_ { \alpha}^\ell \in \mathbb {R}^{d_ { \alpha}  \times 3d}$ are learnable parameters that enable the attention mechanism to adapt to different structural patterns.

Then, for each entity pair $(e_q, e_a)$, we aggregate messages from all incoming edges of $e_a$ to update its pair-wise representation:
\begin{equation}
\resizebox{\linewidth}{!}{$
   \bm h^{\ell}_{e_y|q} = \texttt{AGGREGATE}(\bm h_{e_x|q}^{\ell -1} \cup \{\bm m^{\ell}_{e_x,r,e_y} | (e_x, r, e_y) \in \mathcal{N}(e_y)\}),
    $}
\end{equation}
where $\mathcal{N}_e(e_y)$ denotes the neighboring edge of $e_y$. We specify the \texttt{AGGREGATE} function to be \texttt{sum}, \texttt{mean}, or \texttt{max}.

\paragraph{Overall Path Encoding.}
After $L$ layers of the above message propagation, we obtain the final pair-wise representation $\bm h^{L}_{r_q}(e_q, e_a)$ for each entity pair $(e_q, e_a)$. This collect all paths of up to length $L$ connecting $e_q$ to $e_a$ under query relation $r_q$, thus encoding the reasoning paths from the query entity $e_q$ to any candidate $e_a$. The pair-wise representation can then be used for downstream scoring functions, such as
\begin{equation}
s_L(q, e_a)
=
\bm{w}_L^\top \, \bm h^L_{e_a|q},
\end{equation}
where $\bm{w}_L$ is a trainable parameter vector. By comparing $s_L(q,e_a)$ among different candidate entities $e_a\in\mathcal{V}$, the model predicts which entity is the correct answer for the query $(e_q, r_q, ?)$.

Overall, this recursive path encoding process effectively captures the structural and query-relevant information from multiple-length paths, leveraging dynamic message passing steps that incorporate the query relation to focus on the most relevant edges and intermediates for knowledge graph reasoning.

\subsection{Supplementary theoretical analysis of mixture of length Experts} \label{Supplementary theoretical analysis of mixture of length Experts}
\subsubsection{Design details of $c_q$}  \label{Design details of $c_q$}
The design of query context representation $\bm c_q$ is motivated by the need to capture both structural patterns and semantic information in knowledge graph reasoning. We construct $\bm c_q$ by combining two essential components:

First, $\bm h_{r_q}^{L_{\min}}(e_q,e_q)$ encodes the local structural information around the query entity $e_q$ within the minimum path length $L_{\min}$. This term captures how the query entity connects to its neighborhood, providing crucial information about the local graph topology that can guide length selection. By using the self-loop representation ($e_q$ to $e_q$), we ensure that the structural encoding is centered on the query entity's perspective. Second, $\bm h_{r_q}$ represents the learnable embedding of the query relation, which encodes the semantic requirements of the reasoning task. Different relations may require different reasoning depths - for instance, direct relations like \textit{spouse\_of} typically need shorter paths than indirect relations like \textit{colleague\_of\_friend}. Including $\bm h_{r_q}$ allows the model to adapt its length selection based on the semantic nature of the query relation.

The combination of these components through an MLP enables non-linear interaction between structural and semantic information:
\begin{equation}
    \bm c_q = \text{MLP}(\underbrace{\bm h_{r_q}^{L_{\min}}(e_q,e_q)}_{\text{structural info}}  \Vert \underbrace{\bm h_{r_q}}_{\text{semantic info}}) \in \mathbb{R}^{d}.
\end{equation}
This design principle ensures that length selection is informed by both the local graph structure around $e_q$ and the semantic requirements of relation $r_q$, enabling more effective personalization of path exploration strategies.

\subsubsection{Design details of gating function} \label{Design details of gating function}
As illustrated in Appendix~\ref{Length Distributions}, we prove through the analysis of experimental datasets that the answer to the query entity is generally near its neighborhood and does not involve a very long path length. Therefore, we designed a layer-wise binary gating function to control the model to tend to choose a smaller number of layers. 

Specifically, Let $\mathcal{V}^\ell$ be the set of entities $e_x$ reachable from $e_q$ within $\ell$ steps, we collect all the pair-wise representations that can be reached within $\ell$ steps from $e_q$ to obtain the distribution characteristics of the paired path representations of length $L$ in the system as: $\mathbf{H}^\ell_{r_q}(e_q) = [\bm h_{r_q}^L(e_q, e_x)]_{e_x \in \mathcal{V}^\ell} \in \mathbb{R}^{|\mathcal{V}^{\ell}| \times d}$. During training, we employ a differentiable statistics-based gating function that evaluates path quality based on layer-wise feature distributions:
$g_{\ell} = \text{GumbelSigmoid}(\text{MLP}([\mu_\ell, \sigma_\ell], \tau))$, where $\mu_\ell$ and $\sigma_\ell$ capture the distribution of representation matrix $\mathbf{H}^\ell_{r_q}(e_q)$  along paths of length $\ell$. By incorporating the Gumbel-Sigmoid transformation, we introduce a natural bias towards shorter paths while maintaining differentiability:
\begin{equation}
\text{GumbelSigmoid}(\bm x, \tau) = \sigma\big((\bm x \!+\! \text{GumbelNoise})/\tau\big).
\end{equation}

The added Gumbel noise and sigmoid activation create a statistical tendency to favor paths with lower length counts, as shorter paths typically contain enough stable correct answer information.  During inference, we further strengthen this preference through a deterministic truncation strategy:
\begin{equation}
g_{\ell} = \begin{cases} 
0,  & \text{if } |\text{CV}_{\ell}| > T \text{ and } {\ell} \geq L/2 \\
1, & \text{otherwise}
\end{cases}
\end{equation}
where $\text{CV}_{\ell} = \sigma_{\ell}/\mu_{\ell}$ represents the coefficient of variation of the representation matrix $\mathbf{H}^\ell_{r_q}(e_q)$ at length $\ell$, and $T$ is a predefined threshold controlling the aggressiveness of path truncation. This length control mechanism defined as: $\bm h_{new}^{\ell} \leftarrow g_l \cdot \bm h^\ell$ (If $g_l=0$, stop computing) enables the model to systematically prefer shorter paths when they provide sufficient evidence for reasoning.

\subsection{Mixture of Pruning Experts Implementation Supplement} \label{mixture of pruning experts implementation supplement}

 \paragraph{Mixture of Experts.} The MoE framework consists of multiple specialized expert models and a gating mechanism that dynamically selects appropriate experts. Formally, given input $x$, the output of an MoE system can be written as: $y = \sum_{i=1}^n g_i(x) o_i(\bm x)$, where $\bm o_i(\bm x)$ is the output of the $i$-th expert, and $g_i(\bm x)\in \mathcal G(x)$ ($\sum_{i=1}^n g_i(\bm x) = 1$) is the gating weight that determines the contribution of each expert. 
 Following \cite{shazeer2017outrageously}, we employ a noise-enhanced gating mechanism 
  where expert selection is computed as: 
 \begin{equation}
 \resizebox{\linewidth}{!}{$
     \mathcal G(\bm x) = \text{Softmax}\big(\text{TopK}_k(Q(\bm x) \!+\! \epsilon \cdot \text{Softplus}(\bm x\bm W_n))/\tau\big),
 $}
 \label{gating_function}
 \end{equation}
 where $Q(\bm x)$ is the score for total 
 experts, $\tau$ is a temperature parameter, $\epsilon \sim \mathcal{N}(0,1)$ is the Gaussian noise that encourages diverse expert selection and $\bm W_n$ is trainable parameter that learn noise scores.

\subsubsection{Design details of $c^\ell_v$}
As defined in Appendix~\ref{Design details of $c_q$}, the query context representation $\bm c^\ell_v$ is defined as:
\begin{equation}
    \bm c^\ell_v=\text{MLP}(\bm h_{r_q}^{\ell -1}(e_q,e_a) \Vert \bm h_{r_q}) \in \mathbb R^{d}.
\end{equation}
The main difference between our pruning expert context $\bm c^\ell_v$ and the length expert context is that the pair-wise path representation $\bm h_{r_q}^{\ell -1}(e_q,e_a)$ in the pruning expert context changes with the path length $\ell$. This is because the path length expert only needs to calculate once when the path reaches $L_{\min}$, while we need to calculate and apply the pruning expert at each length $\ell$.

\subsubsection{Entity score update after pruning}

After we get the compatibility of each pruning expert at length $\ell$ with $\bm Q^\ell(\bm c^\ell_v)$, the gating function $\mathcal G^\ell(\bm c^\ell_v)$ is calculated as defined in Eq.~\eqref{eq:gating}. We argue that we cannot directly define the retained entities based on their original scores, because some entities are retained by only one pruning expert, while others may be retained by multiple pruning experts simultaneously. Moreover, even for entities that are retained by different pruning experts separately, their scores should be influenced by the weights $g^\ell_{\phi_i} \in \mathcal G^\ell(c^\ell_v)$ assigned to those experts.

Subsequently, we update the scores of selected entities through the gating weights $g^{\ell}_{\phi_i} \in \mathcal  G^{\ell}(c_q)$ of the chosen pruning experts by:
\begin{equation}
s_l'(q,e_a)=\begin{cases}
    \sum_{i=1}^{k_{2}} g^{\ell}_{\phi_i} \cdot \left(s_l(q,e_a)\right) , &e_a \in \mathcal V_{\phi}^{\ell} \\
    0, &e_a \notin \mathcal V_{\phi}^{\ell}
\end{cases} 
\end{equation}
Finally, the score of an candidate answer entity $e_a$ at length $\ell$ will be updated as: \begin{equation}
    s_l(q,e_a) \leftarrow s_l'(q,e_a).
\end{equation}
And the updated scores will be used as the score function used in Eq.~\eqref{eq:score-moe}.

\subsubsection{Path exploration strategy based on incremental sampling}
To discover new entities at each layer while preserving those selected in previous layers, we adopt an incremental sampling approach that builds upon our message passing framework. Specifically, let $\mathcal P^{\ell-1}$ denote the set of paths retained up to layer $(\ell{-}1)$ and $e_l \notin \mathcal P^{\ell -1}$. We update the path set as:
\begin{equation}
\mathcal P^\ell = \mathcal P^{\ell-1}\cup \{\big(e_{l-1},\, r_l,\, e_l\big) |\small{e_l \in \mathcal V_{\phi}^{\ell}}\} ,
\end{equation}

where $e_l$ is the neighboring entities of $e_{l-1}$. By preserving previously discovered paths and selectively adding new entities at each layer, this incremental process refines the path length set up to $L$, expanding coverage of relevant entities for reasoning while maintaining the consistency of entities retained by experts in the previous layers.

\subsubsection{The comprehensiveness of three pruning experts}
The design of our three pruning experts --- Scoring Pruning Expert, Attention Pruning Expert, and Semantic Pruning Expert --- constitutes a comprehensive evaluation framework that directly addresses the limitations identified in Section~\ref{Introduction}. These experts operate synergistically to provide complementary perspectives in path evaluation:
\begin{itemize}[leftmargin=*]
\item The Scoring Pruning Expert evaluates the global contribution of entities to the reasoning task, capturing high-level importance patterns across the knowledge graph.
\item The Attention Pruning Expert focuses on local structural patterns by analyzing relation combinations and topological features, effectively identifying meaningful reasoning chains while filtering out irrelevant paths.

\item The Semantic Pruning Expert assesses the thematic coherence between entities and query relations, ensuring selected paths maintain semantic relevance to the reasoning context.
\end{itemize}
Our extensive experimental results validate that this three-expert design achieves an optimal balance between evaluation coverage and pruning effectiveness. The experts work in concert to identify high-quality reasoning paths (e.g., \textit{followed}→\textit{singed}) while effectively filtering out spurious combinations (e.g., \textit{is\_friend\_with}→\textit{directed\_in}). The complementary nature of these experts --- operating across global importance, local structure, and semantic relevance --- creates a robust evaluation framework that comprehensively covers the key aspects of path assessment in knowledge graph reasoning. This thorough coverage makes additional pruning experts not only unnecessary but potentially counterproductive, as they would increase computational overhead without providing substantively new evaluation criteria.

\subsection{Sampling Number Function Design} \label{Sampling number function design}

To effectively control path exploration at different depths, we propose an adaptive sampling strategy through below functions:

\begin{equation}
\resizebox{\linewidth}{!}{$
   K^\ell = \begin{cases}
       K_s + (K_h - K_s) \cdot \sigma(a \cdot (l - l_i/2)), & \ell < l_i \\
       K_l + (K_h - K_l) \cdot (1 - \sigma(a \cdot (l - 3l_i/2))), & \ell \geq l_i
   \end{cases}
   $}
\end{equation}

where $\sigma$ is the sigmoid function and $a$ controls the steepness of the transition. This design addresses several key challenges in path exploration. When using a uniform sampling formula, two critical issues emerge at different stages of exploration: First, in the initial sampling phase, the number of neighbor entities $|e_0|$ may be smaller than the predetermined sampling number $K^0$, making it impossible to achieve the target sampling quantity. Given that the neighborhood size $| \mathcal N_n(e_l)|$ typically grows exponentially with layer depth $L$, restricting the sampling number based on $e_0$ would result in missing many important paths at deeper layers. 

However, we cannot simply increase the sampling number $K^\ell$ indefinitely with path length, as the proportion of noise in the paths tends to increase with depth. This necessitates the introduction of an inflection point layer $l_i$. Once this inflection point is reached, the sampling number should gradually decrease with increasing layer depth to control noise accumulation.

Our formula incorporates three crucial parameters: initial sampling $K_s$, maximum sampling $K_h$, and minimum sampling $K_l$. This design accommodates the initial neighborhood size $|\mathcal V^0|$ while using the maximum and minimum sampling thresholds to dynamically control path retention at different layers. In the early stages ($\ell < l_i$), the sampling number gradually increases from $K_s$ toward $K_h$, allowing for broader exploration. Beyond the inflection point ($\ell \geq l_i$), it decreases from $K_h$ toward $K_l$, focusing on the most relevant paths. The sigmoid function ensures smooth transitions between these phases, while parameter $a$ allows fine-tuning of the transition rate.

This adaptive sampling strategy enables more effective personalized path exploration by balancing the need for comprehensive coverage in early layers with focused path selection in deeper layers, while maintaining robustness to varying neighborhood sizes across different queries.

\section{Supplementary Case Study}\label{Supplementary Case Study}

\begin{figure}[h]
  \centering
    \begin{subfigure}[b]{0.16\textwidth}
    \centering
    \includegraphics[height=1.5cm, width=\textwidth]{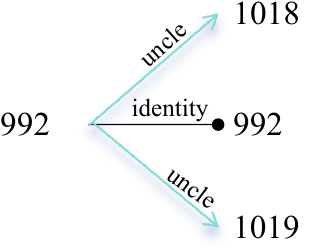}
    \caption{$\mathcal Q (\text{992},\text{uncle},?)$}
    \label{fig:case_d} 
  \end{subfigure}%
  \begin{subfigure}[b]{0.29\textwidth}
    \centering
    \includegraphics[height=1.5cm, width=\textwidth]{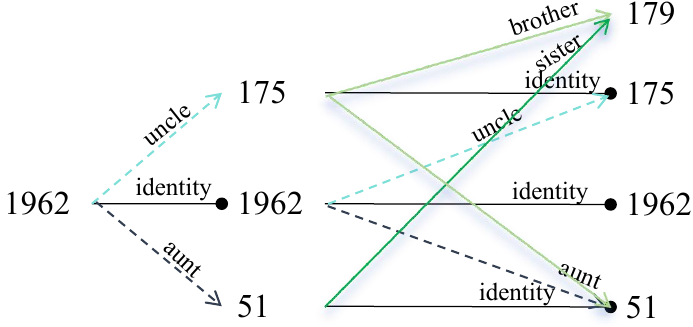}
    \caption{$\mathcal Q (\text{1962},\text{niece},?)$}
    \label{fig:case_c}
  \end{subfigure}%
  \\
  \begin{subfigure}[b]{0.225\textwidth}
    \centering
    \includegraphics[height=1.9cm, width=0.92\textwidth]{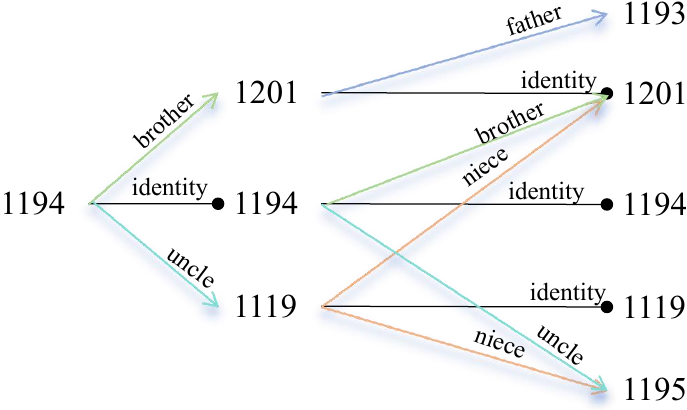}
    \caption{$\mathcal Q(\text{1194},\text{brother},?)$}
    \label{fig:case_study_a} 
  \end{subfigure}
  \begin{subfigure}[b]{0.225\textwidth}
    \centering
    \includegraphics[height=1.9cm, width=0.92\textwidth]{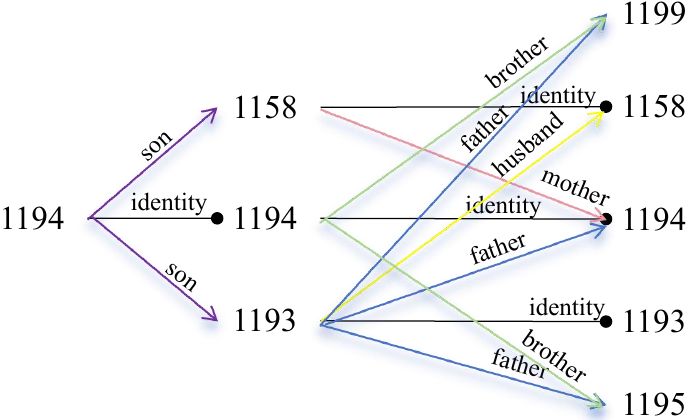}
    \caption{$\mathcal Q(\text{1194},\text{father},?)$}
    \label{fig:case_study_b} 
  \end{subfigure}%
  \caption{Visualization of the transmission path on the family dataset, with dotted lines representing reverse relationships.}
  \label{fig:case_study_visualization}
\end{figure}

To validate our hypothesis that different queries require varying reasoning path lengths, we present two illustrative examples from the Family dataset. For $\mathcal Q(\text{992},\text{uncle},?)$ (Fig.~\ref{fig:case_d}), the desired answer entity emerges at the first length, while for the query $\mathcal Q(\text{1962},\text{niece},?)$ (Fig.~\ref{fig:case_c}), the correct answer entity does not appear until the second length. 
This demonstrates that a single-length exploration is insufficient for certain queries, 
and multi-length exploration can adapt better and save extra computational cost.
Meanwhile, MoKGR also maintain semantic awareness, which is further evidenced in Fig.~\ref{fig:case_study_a} and \ref{fig:case_study_b}, where for entity \textit{1194}, MoKGR selects kinship-appropriate relations: \textit{brother} and \textit{uncle} for brother queries, and \textit{son} and \textit{brother} for father queries, thereby improving both reasoning accuracy and efficiency.

\section{Additional Implementation Details}

\subsection{Personalized PageRank}

\begin{algorithm}[ht]
\setlength{\textfloatsep}{4pt} 
\setlength{\intextsep}{4pt}
\caption{Training Process of MoKGR}
\label{alg:training}
\begin{algorithmic}[1]

\REQUIRE Parameters: Number of length and pruning experts $k_1$ and $k_2$, range of path lengths $[L_{min}, L]$.
\ENSURE Optimized GNN model parameters $\Theta$ and experts model parameters $\mathbb W$.

\STATE // Pre-processing with PPR
\STATE Initialize PPR cache for all entities in $\mathcal{V}$
\FOR{$v \in \mathcal{V}$}
    \STATE Compute PPR scores $\pi_v$ and store in cache
\ENDFOR

\STATE // Training Loop
\WHILE{not converged}
    \STATE Sample a batch of queries $\{(e_q, r_q, e_a)\}$ from $\mathcal{Q}_{tra}$
    \FOR{each query $(e_q, r_q, e_a)$, $\ell \in [1,L]$}
        \STATE // PPR-based subgraph construction
        \STATE $\mathcal{G}_{sub} \gets \text{BuildSubgraph}(e_q, \text{PPRCache})$
         \IF{$\ell==L_{\min}$}
            \STATE // Length Expert Selection
            \STATE Compute context representation $\bm{c}_q$ and expert embedding $\bm E_1$, thus get the compatibility with experts via $\bm Q(\bm c_q) = \bm E_1 \bm c_q+\epsilon\cdot\text{Softplus}(\bm W_n\bm c_q)$;
            \STATE Select Top-$k_1$ length experts from $\ell \in [L_{min}, L]$ via $\bm{Q}(\bm c_q)$ to get weights set $\mathcal G_q$;
        \ENDIF
        \STATE // Pruning Expert Selection at length $\ell$
        \STATE Compute context $\bm{c}_v^\ell$ and expert embedding $\bm E_2$ for pruning;
        \STATE Select Top-$k_2$ pruning experts via $\bm Q^\ell(\bm c^\ell_v)$ to get weights $\mathcal G_v$;
        \STATE Combine selected experts to identify key entities: $\mathcal V_{\phi}^{\ell} = \{\cup_{i\in\text{TopK}_{k_2}(\bm Q^\ell(\bm c^\ell_v))} \mathcal V_{\phi_i}^\ell | \mathcal V_{\phi_i}^\ell = \text{TopK}_{K^\ell}(\phi_i^{\ell}(e_a))\}$;
        \STATE Update path representations for identified entities in $\mathcal{V}_\phi^\ell$;

        \IF{$\ell$ is the selected length expert}
            \STATE Calculate entity scores $s_{\ell}(e_q,r_q,e_a)$ at current length;
            \STATE Update final scores $\Psi(e_a)$ by combining weighted length-specific scores;
            
            \IF{early stopping condition met: $g_b(\ell)=0$}
                \STATE break
            \ENDIF
        \ENDIF
    \ENDFOR
    
    \STATE // Parameter Updates
    \STATE Compute total loss $\mathcal{L}$ combining task and expert balance losses; 
    \STATE Update model parameters $\Theta$ and expert parameters $\mathbb W$ using gradient of $\mathcal{L}$;
\ENDWHILE
\RETURN $\Theta$, $\mathbb W$.
\end{algorithmic}
\end{algorithm}

In this subsection, we will introduce an auxiliary method --- the principle of PPR (Personalized PageRank). On some ultra-large-scale datasets, even after implementing pruning, the computational cost remains enormous since pruning at each layer requires calculations for all entities in that layer. Therefore, we explored whether we could implement pre-pruning functionality that could filter out some less important nodes in advance for ultra-large-scale datasets without trainable parameters (as these would increase computational costs) through simple rules. Based on this, we discovered that Personalized PageRank, an algorithm based on random walks, could effectively accomplish this task. Thus, for ultra-large datasets such as YAGO3-10, we performed preliminary subgraph extraction using PPR in advance. Algorithm~\ref{alg:training} demonstrates the complete algorithmic workflow after activating PPR.
\subsubsection{PPR Sampling methods}
PPR is commonly implemented through the power iteration method, which can be formulated as 
\begin{equation}
    \pi = \alpha \bm \pi \bm P + (1-\alpha)\bm v    ,
\end{equation}
where $\pi$ represents the PPR vector, $P$ denotes the row-normalized adjacency matrix, $v$ is the personalization vector with the initial node having value 1 and others 0, and $\alpha \in (0,1)$ is the damping factor (typically set to 0.85). The iteration continues until convergence or reaching a maximum number of steps. This approach effectively captures the probability distribution of random walks with restart, where at each step, the walk either continues to a neighboring node with probability $\alpha$ or teleports back to the initial node with probability $(1-\alpha)$. The resulting PPR scores indicate the relative importance of nodes with respect to the initial node, making it particularly useful for local graph analysis and node ranking tasks.
Our experiments demonstrate that PPR-based preliminary exploration significantly enhances the efficiency and effectiveness of path-based reasoning. Without such preliminary filtering, large-scale KGs reasoning often encounters memory constraints due to the exponential growth of potential paths. The computational complexity of PPR's random walk-based approach is substantially lower than that of GNN operations, making it an efficient choice for initial path exploration. Furthermore, the paths preserved through PPR's preliminary exploration help subsequent message passing and pruning mechanisms focus on truly relevant reasoning paths while filtering out noise.

\paragraph{PPR Calculation} 
We compute and cache global PPR scores for all entities in the knowledge graph to capture their overall importance and connectivity patterns. The PPR score for entity $v$ is defined as:
\begin{equation}
   \bm  \pi_{e_v}(v) = \alpha \cdot \mathbf{e}_v + (1 - \alpha) \cdot \mathbf{P}^\top \bm \pi_{e_v}(v),
\end{equation}
where $\mathbf{e}_v$ is the indicator vector for entity $v$, and $\mathbf{P}$ is the transition probability matrix of the graph. For efficient computation on large-scale graphs, we implement a GPU-accelerated iterative method that approximates $\bm \pi_v$ for all entities until convergence or reaching a maximum iteration limit.

\paragraph{Query-Specific Path Exploration}
For each query entity $e_q$ in a batch, we evaluate each entity $e_v$'s importance by aggregating its PPR scores across all query entities:
\begin{equation}
\text{score}(e_v)=\sum_{e_q\in {\text{batch}}} \bm \pi_{e_q}(e_v),
\end{equation}
where $\bm \pi_{e_q}(v)$ represents the personalized PageRank score of entity $e_v$ when using $e_q$ as the starting node.

Based on these scores, we identify promising paths by selecting entities in descending order of importance until reaching a predefined exploration scope. All subsequent reasoning is performed primarily along these preliminarily identified paths. While we explore different path sets for distinct queries, the PPR scores are pre-computed only once, ensuring computational efficiency while maintaining personalization for each query.

\subsection{Loss function calculation supplement} \label{Loss function calculation supplement}
\paragraph{Experts Balance Loss} \label{Experts Balance Loss}
To ensure balanced and effective path exploration, we introduce several regularization terms that prevent the model from overly relying on specific exploration strategies or experts. This addresses the potential ``winner takes all'' problem \cite{lepikhin2020gshard} where a single expert might dominate the path exploration process. Let $\mathcal{C}$ denote the set of query context vectors in the current batch, thus $\mathcal C_q$ and $\mathcal C_v$ contains all $\bm c_q$ and $\bm c_v$ in the batch respectively. 

First, we introduce importance loss $\mathcal L_{\text{importance}}(\mathcal C)$ as length-level balance loss $\mathcal L_l(\mathcal C_q)$ and pruning $\mathcal L_p(\mathcal C_v)$ balance loss to encourage diverse path lengths and balanced pruning strategy utilization:
\begin{equation}
    \begin{aligned}
        &\text{Importance}(\mathcal C) = \sum_{\bm c \in \mathcal C}\sum_{ g \in \mathcal G(\bm c)} g , \\
        &\mathcal L_{\text{Importance}}(\mathcal C) = \text{CV}(\text{Importance}(\mathcal C))^2 ,
    \end{aligned}
\end{equation}

where $g \in \mathcal G(\bm c)$ is the output of experts' gating mechanism calculated as Eq.~\eqref{eq:gating}, $\text{CV} (\bm X)= \sigma(\bm X)/\mu(\bm X)$ represents the coefficient of variation of input $\bm X$. The importance loss hence measures the variation of importance scores, enforcing all experts to be “similarly important".
While the importance score enforces equal scoring among the experts, there may still be disparities in the load assigned to different experts.
To address this, we additionally introduce a load balancing penalty for length experts to prevent overloading of specific experts.
Specifically, let $P(\bm c_{q},\ell)$ denote the probability that length-$\ell$ expert is selected (i.e., $g_q(\ell) \neq 0$ :
$P(\bm c_{q_i},\ell) = \text{Pr}(Q(\bm c_q)_\ell > \text{kth\_ex}(\bm Q(\bm c_{q}),k,\ell))$
where $\text{kth\_ex}(\cdot)$ returns the $k$-th largest expert score excluding the expert itself.

We give a simplified solution to this formula as:
\begin{equation}
    P(\bm c_{q}, \ell) = \Phi\left(\frac{\bm c_q \bm W_g - \text{kth}\_{\text{ex}}(\bm Q(\bm c_q), k, \ell)}{\text{Softplus}(\bm  c_q \bm W_n)}\right),
\end{equation}
where $W_g \in \mathbb R^{d \times H}, W_n \in \mathbb R^{d \times H}$ are learnable weights and $\Phi$ is the CDF of standard normal distribution. 

The length load balance loss is then defined as:
\begin{equation}
    \mathcal L_{\text{load}}(\mathcal C)= \text{CV}(\sum_{\bm c_q \in \mathcal C}\sum_{ p \in P(\bm c_{q},\ell)}p)^2 ,
\end{equation}
where $p$ is the node-wise probability in the batch.

\section{Complexity Analysis}

Let $|\mathcal{V}^{\ell}|$ denote the number of entities and $|\mathcal{E}^{\ell}|$ denote the number of edges between entities at length $\ell-1$ and $\ell$ in the knowledge graph. Traditional GNN-based methods like NBFNet require $O(\sum_{\ell=1}^L|\mathcal{V}^{\ell}| \cdot |\mathcal{E}^{\ell}|)$ operations to process all paths up to length $L$. In contrast, MoKGR reduces the computational cost through two key mechanisms: (1) The layer-wise binary gating function enables early stopping of unnecessary path explorations, reducing the effective path length from $L$ to an 
adaptive length $L_a$ ($ L_a \leq L$); (2) The mixture of pruning experts first evaluates all entities with complexity $O(|\mathcal{V}^{\ell}|)$ at each layer $\ell$, and then retains only $K^\ell$ most promising entities where $K^\ell \ll |\mathcal{V}^{\ell}|$ and $|\mathcal E_{K^\ell}|\ll|\mathcal E^\ell|$. Consequently, MoKGR achieves an overall operations of $O(\sum_{\ell=1}^\mathcal L (|\mathcal{V}^{\ell}| \cdot k_2 + K^\ell \cdot |\mathcal{E}_{K^\ell}|))$, where $|\mathcal{V}^{\ell}| \cdot k_2$ denotes the number of operations caused by the retained pruning expert calculation at length ${\ell}$, which remains substantially more efficient than traditional methods for large-scale knowledge graphs since $K^\ell \cdot |\mathcal{E}_{K^\ell}| \ll |\mathcal{V}^{\ell}| \cdot |\mathcal{E}^{\ell}|$ and $k_2 \ll |\mathcal E^\ell|$.

\section{Theoretical Analysis} \label{Theoretical Analysis}

In this appendix, we present a theoretical analysis of the MoKGR framework, including convergence guarantees, optimality of path selection, preservation properties of pruning mechanisms, and formal complexity bounds.

\subsection{Convergence Properties of MoKGR}

\begin{theorem}[Convergence of MoKGR]
Under appropriate conditions on the expert selection probabilities and learning rates, the MoKGR algorithm converges to a local optimum of the loss function.
\end{theorem}

\begin{proof}
Let's define the loss function for MoKGR as:
\begin{equation}
L = L_{task} + \lambda_1(L_l + L_p) + \lambda_2L_{load}
\end{equation}

Where $L_{task}$ is the task-specific loss, $L_l$ and $L_p$ are the length and pruning expert importance losses, and $L_{load}$ is the load balancing loss.

The gradient descent update for the parameters $\Theta$ of the model at iteration $t$ is:
\begin{equation}
\Theta_{t+1} = \Theta_t - \eta_t \nabla_\Theta L(\Theta_t)
\end{equation}

where $\eta_t$ is the learning rate at iteration $t$.

For convergence, we need to show that:
\begin{enumerate}[leftmargin=*]
    \item The loss function $L$ is bounded below.
    \item The gradient $\nabla_\Theta L(\Theta_t)$ is Lipschitz continuous.
    \item The learning rate satisfies $\sum_{t=1}^{\infty} \eta_t = \infty$ and $\sum_{t=1}^{\infty} \eta_t^2 < \infty$.
\end{enumerate}

First, observe that $L_{task}$ is bounded below by 0 (as it's a negative log-likelihood loss). The expert balance losses $L_l$, $L_p$, and $L_{load}$ are all non-negative as they are based on squared coefficients of variation. Therefore, $L$ is bounded below.

For Lipschitz continuity, the scoring function $\Psi(e_a) = \sum_{l\in A} g_q(l) \cdot s_l(q, e_a)$ is a linear combination of expert outputs, each of which is bounded and Lipschitz continuous due to the bounded nature of the message passing operations and the softmax gating function.

Given a decreasing learning rate schedule $\eta_t = \frac{\eta_0}{\sqrt{t}}$, we have:
\begin{equation}
\sum_{t=1}^{\infty} \eta_t = \eta_0 \sum_{t=1}^{\infty} \frac{1}{\sqrt{t}} = \infty
\end{equation}
\begin{equation}
\sum_{t=1}^{\infty} \eta_t^2 = \eta_0^2 \sum_{t=1}^{\infty} \frac{1}{t} < \infty
\end{equation}

Therefore, by the convergence theorem for stochastic gradient descent with Lipschitz continuous gradients, the algorithm converges to a local optimum of the loss function.
\end{proof}

\subsection{Optimality of Adaptive Path Length Selection}

\begin{theorem}[Optimality of Path Length Selection]
The adaptive path length selection mechanism in MoKGR minimizes the expected reasoning error given a computational budget constraint.
\end{theorem}

\begin{proof}
Let $E(l, q)$ be the expected reasoning error when using paths of length up to $l$ for query $q$. Let $C(l)$ be the computational cost of exploring paths of length $l$.

The problem can be formulated as:
\begin{align}
\min_{\{w_l\}_{l=L_{min}}^L} &\sum_{q \in Q} \sum_{l=L_{min}}^L w_l(q) E(l, q)\\
\text{subject to } &\sum_{q \in Q} \sum_{l=L_{min}}^L w_l(q) C(l) \leq B\\
&\sum_{l=L_{min}}^L w_l(q) = 1, \forall q \in Q\\
&w_l(q) \geq 0, \forall l, q
\end{align}

where $w_l(q)$ is the weight assigned to path length $l$ for query $q$, and $B$ is the computational budget.

The adaptive length selection mechanism in MoKGR computes weights as:
\begin{equation}
g_q(l) = \frac{\exp([Q(c_q)]_l/\tau)}{\sum_{l' \in A} \exp([Q(c_q)]_{l'}/\tau)}
\end{equation}

where $[Q(c_q)]_l$ is the compatibility score between query $q$ and path length $l$.

The key insight is that the compatibility score $[Q(c_q)]_l$ learns to correlate with the negative expected error $-E(l, q)$ through training. This occurs because queries that benefit more from specific path lengths will have higher accuracy when those lengths are selected, leading to lower task loss.

The noise term $\epsilon \cdot \text{Softplus}(W_n c_q)$ enables exploration of different length combinations, allowing the model to discover the optimal path length distribution for each query type.

The binary gating function $g_b(l)$ enforces the budget constraint by encouraging shorter paths when they provide sufficient evidence.

As training progresses, the model learns to assign higher weights to path lengths that minimize the expected error for each query while respecting the computational budget constraint.

Therefore, the adaptive path length selection mechanism converges to the optimal weighting that minimizes the expected reasoning error given the computational constraints.
\end{proof}

\subsection{Preservation Properties of Pruning Mechanism}

\begin{theorem}[Preservation of Optimal Paths]
Under certain conditions, the mixture of pruning experts ensures that the optimal reasoning path for answering a query is preserved with probability at least $1-\delta$.
\end{theorem}

\begin{proof}
Let $P(e_q, e_a)$ be the set of all paths connecting query entity $e_q$ to potential answer entity $e_a$. Let $p^* \in P(e_q, e_a)$ be the optimal path that provides the strongest evidence for answering the query.

Let $V_l$ be the set of entities at distance $l$ from $e_q$, and let $V_l^{\phi}$ be the subset selected by the pruning mechanism. For the optimal path $p^*$ to be preserved, all entities along $p^*$ must be included in the selected subsets.

Let $e^*_l$ be the entity at distance $l$ along the optimal path $p^*$. We need to show that:
\begin{equation}
\Pr(e^*_l \in V_l^{\phi}) \geq 1 - \delta
\end{equation}
for some small $\delta > 0$.

The MoKGR pruning mechanism selects entities based on the union of top-$K_l$ entities according to different pruning experts:
\begin{equation}
\resizebox{0.99\linewidth}{!}{$
V_l^{\phi} = \{\cup_{i \in \text{TopK}_{k_2}(Q^l(c_v^l))} V_l^{\phi_i} | V_l^{\phi_i} = \text{TopK}_{K_l} (\phi_l^i(e_a))\}
$}
\end{equation}

For entity $e^*_l$ to be excluded from $V_l^{\phi}$, it must be excluded by all selected pruning experts. The probability of this happening is:
\begin{equation}
\Pr(e^*_l \notin V_l^{\phi}) = \Pr\left(\bigcap_{i \in \text{TopK}_{k_2}(Q^l(c_v^l))} \{e^*_l \notin V_l^{\phi_i}\}\right)
\end{equation}

Since our three pruning experts evaluate different aspects of path quality (scoring, attention, and semantic relevance), they are designed to be complementary. The optimal path $p^*$ should score highly on at least one of these dimensions.

Let's denote by $\rho_i$ the probability that entity $e^*_l$ is not selected by pruning expert $i$. Then:
\begin{equation}
\Pr(e^*_l \notin V_l^{\phi}) \leq \prod_{i \in \text{TopK}_{k_2}(Q^l(c_v^l))} \rho_i
\end{equation}

For the optimal path, at least one of the experts should rank $e^*_l$ highly. Let's say that for the best-matched expert $i^*$, we have $\rho_{i^*} \leq \epsilon$ for some small $\epsilon > 0$.

Then:
\begin{equation}
\Pr(e^*_l \notin V_l^{\phi}) \leq \epsilon \cdot \prod_{i \in \text{TopK}_{k_2}(Q^l(c_v^l)), i \neq i^*} \rho_i \leq \epsilon
\end{equation}

Therefore:
\begin{equation}
\Pr(e^*_l \in V_l^{\phi}) \geq 1 - \epsilon
\end{equation}

By setting $\delta = L\epsilon$ where $L$ is the maximum path length, and applying the union bound, we can show that the entire optimal path is preserved with probability at least $1 - \delta$.

This proves that the mixture of pruning experts preserves the optimal reasoning path with high probability.
\end{proof}

\subsection{Information Theoretic Analysis of Adaptive Path Selection}

\begin{theorem}[Information Gain of Adaptive Path Selection]
The adaptive path length selection mechanism in MoKGR maximizes the expected information gain about the answer entity while respecting computational constraints.
\end{theorem}

\begin{proof}
Let $H(E_a | e_q, r_q)$ be the entropy of the answer entity distribution given query $(e_q, r_q, ?)$. Let $I(E_a; P_l | e_q, r_q)$ be the mutual information between the answer entity and paths of length $l$ given the query.

The information gain from exploring paths of length $l$ is:
\begin{align}
IG(l) &= H(E_a | e_q, r_q) - H(E_a | P_l, e_q, r_q)\\
&= I(E_a; P_l | e_q, r_q)
\end{align}

The expected information gain from the adaptive path length selection is:
\begin{equation}
E[IG] = \sum_{l=L_{min}}^L g_q(l) \cdot I(E_a; P_l | e_q, r_q)
\end{equation}

where $g_q(l)$ is the weight assigned to path length $l$ for query $(e_q, r_q, ?)$.

The goal of the adaptive path length selection mechanism is to maximize this expected information gain subject to computational constraints:
\begin{align}
\max_{g_q} &\sum_{l=L_{min}}^L g_q(l) \cdot I(E_a; P_l | e_q, r_q)\\
\text{subject to } &\sum_{l=L_{min}}^L g_q(l) \cdot C(l) \leq B\\
&\sum_{l=L_{min}}^L g_q(l) = 1, g_q(l) \geq 0
\end{align}

where $C(l)$ is the computational cost of exploring paths of length $l$, and $B$ is the computational budget.

The compatibility score $[Q(c_q)]_l$ in MoKGR can be interpreted as an estimate of the information gain $I(E_a; P_l | e_q, r_q)$. By learning to assign higher weights to path lengths with higher information gain, MoKGR effectively solves the optimization problem~\cite{huang2025mldebugging}.

The layer-wise binary gating function further enforces the computational constraint by stopping path exploration when the expected additional information gain does not justify the computational cost.

Therefore, the adaptive path length selection mechanism in MoKGR maximizes the expected information gain about the answer entity while respecting computational constraints.
\end{proof}

\end{document}